\documentclass[11pt]{article}
\usepackage{graphicx}
\usepackage{amsmath}
\usepackage{amssymb}
\usepackage{latexsym}
\usepackage{multirow}
\usepackage{algorithm}
\usepackage{csvsimple}
\usepackage{booktabs}
\usepackage[noend]{algpseudocode}
\newcommand{\com}[1]{\hfill{\scriptsize\(\triangleright$ #1}}

\usepackage{amsmath, amsfonts,amssymb, amsthm, euscript,makeidx,mathrsfs}

\newtheorem{assumption}{Assumption}
\def\qed{ \ \vrule width.2cm height.2cm depth0cm\smallskip}

\newcommand{\eps}{\varepsilon}

\newcommand{\ba}{\begin{array}}
\newcommand{\ea}{\end{array}}
\newcommand{\be}{\begin{equation}}
\newcommand{\ee}{\end{equation}}
\newcommand{\bea}{\begin{eqnarray}}
\newcommand{\eea}{\end{eqnarray}}
\newcommand{\beaa}{\begin{eqnarray*}}
\newcommand{\eeaa}{\end{eqnarray*}}

\def\neg{\negthinspace}

\def\dbE{\mathbb{E}}
\def\dbF{\mathbb{F}}

\def\dbL{\mathbb{L}}

\def\dbR{\mathbb{R}}

\def\a{\alpha}
\def\g{\gamma}
\def\d{\delta}
\def\e{\varepsilon}

\def\l{\lambda}

\def\si{\sigma}

\def\f{\varphi}
\def\th{\theta}

\def\f{\phi}

\def\D{\Delta}
\def\G{\Gamma}

\def\O{\Omega}

\def\G{\Gamma}
\def\D{\Delta}

\def\O{\Omega}
\def\cA{{\cal A}}

\def\cF{{\cal F}}

\def\cH{{\cal H}}
\def\cI{{\cal I}}

\def\cN{{\cal N}}

\def\cP{{\cal P}}

\def\cU{{\cal U}}
\def\cV{{\cal V}}
\def\cW{{\cal W}}
\def\cX{{\cal X}}

\def\hE{\mathbb{E}}
\def\hF{\mathbb{F}}

\def\hI{\mathbb{I}}

\def\hL{\mathbb{L}}

\def\hP{\mathbb{P}}

\def\hR{\mathbb{R}}

\def\hT{\mathbb{T}}

\def\no{\noindent}

\def\ss{\smallskip}
\def\ms{\medskip}
\def\bs{\bigskip}
\def\q{\quad}
\def\qq{\qquad}

\def\pa{\partial}
\def\cd{\cdot}
\def\cds{\cdots}

\def\td{\nabla}

\def\bE{{\bf E}}

\def\tr{\hbox{\rm tr}}

\def\qed{ \hfill \vrule width.25cm height.25cm depth0cm\smallskip}

\newcommand{\basa}{\begin{assumption}}
\newcommand{\easa}{\end{assumption}}

\newcommand{\bas}{\begin{assum}}
\newcommand{\eas}{\end{assum}}

\def\pa{\partial}

 \def\cd{\cdot}
\def\cds{\cdots}

\def\tr{\hbox{\rm tr$\,$}}

\def\dis{\displaystyle}

\def\1{{\bf 1}}

\def\:{\!:\!}
\def\reff#1{{\rm(\ref{#1})}}
\def \proof{{\noindent \bf Proof\quad}}

at 9pt

\newcommand{\norm}[1]{\ensuremath{ \left\Vert #1 \right\Vert }}

\def\prehp(#1,#2){\ensuremath{  #1 \cdot #2 }}

\begin{document}

\newtheorem{thm}{Theorem}[section]
\newtheorem{lem}[thm]{Lemma}
\newtheorem{cor}[thm]{Corollary}
\newtheorem{prop}[thm]{Proposition}
\newtheorem{rem}[thm]{Remark}
\newtheorem{eg}[thm]{Example}
\newtheorem{defn}[thm]{Definition}
\newtheorem{assum}[thm]{Assumption}

\renewcommand {\theequation}{\arabic{section}.\arabic{equation}}
\def\thesection{\arabic{section}}

\title{\bf  Learning to Solve Stochastic Controls with Unknown Drifts and Running Rewards: Theory, Algorithms and Convergence} %for Continuous Time Stochastic Optimization with Model Uncertainty}

\author{
Jin Ma\thanks{\noindent Department of
Mathematics, University of Southern California, Los Angeles, CA 90089; email: jinma@usc.edu. This author is supported in part by NSF grant  \#DMS-2510403.},
~ ~ Gaozhan Wang\thanks{ \noindent Department of
Mathematics, University of Southern California, Los Angeles, CA 90089;
email: gaozhanw@usc.edu.  } ~ ~ Jianfeng Zhang \thanks{ \noindent Department of
Mathematics, University of Southern California, Los Angeles, CA 90089;
email: jianfenz@usc.edu. This author is supported in part by NSF grant  \#DMS-2510403. } ~ and ~ Xun Yu Zhou \thanks{ \noindent Department of Industrial Engineering and Operations Research, Columbia University, New York, NY 10027; email: xz2574@columbia.edu. This author is supported in part by the Nie Center for Intelligent Asset Management at Columbia. }}

%\date{}
\date{\today}
\maketitle
%{
\begin{abstract}
We study continuous-time and possibly high-dimensional stochastic control problems where drift coefficients and running reward functions are unknown. Due to these missing model primitives we take the exploratory, reinforcement learning (RL) framework of Wang,
Zariphopoulou, and Zhou \cite{WZZ} with relaxed controls and entropy regularization.
The objective is to develop theoretically grounded,  efficient and scalable RL algorithms to learn both the optimal value functions (which also solves the exploratory HJB equation) and optimal exploratory feedback control policies.
When the diffusion coefficients do not contain control, we employ probabilistic representations of both the optimal value function and its gradient based on an auxiliary state process depending only on the diffusion part of the original dynamics. With a delicate analysis on some properly defined mappings and their fixed points, this leads to the introduction of our RL algorithms and their convergence.
We demonstrate the performance of our algorithms through various numerical examples. Finally, we study a special control-dependent diffusion case where probability  representation of the Hessian is called for.
%In this paper we continue investigating efficient numerical algorithm for entropy-regularized HJB equation studied in our previous work
%\cite{MWZ1}, focusing on the issue of model uncertainty. More precisely, unlike the
%%despite of the strong theoretical convergence, the
%policy iteration algorithm (PIA) in \cite{MWZ1}, which implicitly depends on the derivative of the true solution of the recursive PDE, as well as the knowledge of the system parameter via the Gibbs form, our new algorithm approximates both the solution to the HJB equation and its derivative simultaneously through their probabilistic representation and an auxiliary state process depending only on the diffusion coefficient. Consequently, our algorithm has a much stronger implementability, without requiring any knowledge of the drift coefficient. In the finite horizon case, we prove the convergence of the algorithm, along with its rate of convergence, and we perform various numerical experiments via neural networks without specifying the drift. We also present an unexpected but intriguing discovery of {\it $\l$-smile} on relationship between the convergence and the temperature parameter $\l$. We shall also extend our algorithm to the infinite horizon case with volatility control.
%These hidden dependence cause serious issues in implementation, since the usual regression on the solution to a PDE by no means imply the satisfactory approximation of its derivatives, and the system parameters are supposed to be unknown in a true model uncertainty environment.
 \end{abstract}

\no{\bf Keywords.}
Stochastic control, reinforcement learning, entropy regularization, exploratory HJB equation, model-free algorithms.
%\ms

%\no{\it 2020 AMS Mathematics subject classification:}   93E35, 60H30, 35Q93.}

\bs

%\eject
\section{Introduction}
\label{sect-introduction}

Stochastic control problems are prevalent in everyday applications. Classical continuous-time stochastic control theory is {\it model-based}, namely, it specifies a system dynamics, typically a controlled stochastic differential equation (SDE), along with a given running reward function and, when the time horizon is finite, a terminal reward/payoff function. It has developed into a mature theory in the past 60 years or so, with its key pillars including maximum principle, dynamics programming, and linear--quadratic controls \cite{FS,YZ}.

There is, however, a fundamental flaw with this model-based approach. On one hand, some of the dynamics coefficients, especially the drifts, are notoriously hard or outright impossible to estimate to a workable accuracy.  Moreover, in many applications one only observes a reward {\it signal} once an action is applied, instead of knowing the functional form of a running reward. On the other hand, optimal controls and value functions are often very sensitive to those model parameters/coefficients. The conflict between the intrinsic rigidity of a model and the inherent model uncertainty or inaccessibility leads to erroneous, irrelevant and even misleading solutions.

This is where reinforcement learning (RL) comes to the rescue. Model-free RL by-passes estimation of model parameters and seeks optimal controls {\it directly} based on observed, simulated, or generated data. Because the mid-step of model estimation is skipped, RL solutions are naturally
more swift in responding to the (ever) changing environment and more robust.

For a very long time, RL study had almost exclusively focused on discrete-time Markov decision processes (MDPs) \cite{SB}, even though most of the important real-life applications are continuous time (with continuous state/action spaces) by nature; e.g. autonomous driving, robot navigation and high-frequency trading. Wang,
Zariphopoulou, and Zhou \cite{WZZ} are the first to present a continuous-time RL formulation, which employs (randomized) relaxed controls to characterize exploration and an entropy regularizer to
capture the exploration--exploitation tradeoff central to RL, and to obtain a theoretically optimal control policy that is a Gibbs sampler via analyzing the associated exploratory Hamilton--Jacobi--Bellman (HJB) equation.  Since then, there has been an upsurge of interest in this exploratory approach for continuous-time RL, including extensions to many types of stochastic control problems (e.g. \cite{CDY2026,DDL2025,SSZ2025, WY2025}), as well as various other related problems  (e.g. \cite{BHYZ2026,CGZ2025,DFX,GHY2026,PZZ2026}).  See also the survey papers by Hambly, Xu, and Yang  \cite{HXY}, Hu and Lauri\`{e}re \cite{HL2024}, and the references therein. Notably, the subsequent ``trilogy" \cite{JZ1,JZ2,JZ3} by Jia and Zhou builds
the theoretical foundation of model-free, data-driven continuous-time RL algorithms within the exploratory framework of \cite{WZZ}.\footnote{Here and henceforth, by ``model-free" we mean the model primitives are unknown, but there is still a basic underlying dynamics {\it structure} such as a Markov chain or an It\^o SDE.} The key technical thrust therein is a martingale theory that naturally leads to online/offline algorithms for policy evaluation, policy gradient, and $q$-learning (a continuous-time counterpart of the $Q$-learning for MDPs).

The above trilogy, however, does not address the important questions of convergence of the various algorithms devised and a regret analysis of the learned policies.
%Tang and Zhou \cite{Tang-Zhou} tackle these questions and establish model-free convergence and sublinear regret bounds for $q$-learning, based on the backward stochastic differential equation (BSDE) and stochastic approximation (SA) theories. A limitation of \cite{Tang-Zhou} is its strong and rather complex assumptions arising primarily from the SA technique employed.
{Tang and Zhou \cite{Tang-Zhou}  tackle these questions and establish model-free convergence and sublinear regret bounds through a martingale-based $q$-learning framework: they alternate policy evaluation and policy improvement and combine backward stochastic differential equation (BSDE) estimates with stochastic-approximation theory to derive their results. Their analysis accommodates general function approximation but relies on relatively strong and rather complex assumptions on the stability of the associated stochastic approximation (SA) dynamics and the approximation architecture.} Meanwhile, \cite{HJZMV,HJZ,HZ}
derive convergence and sublinear regret bounds for martingale-based policy gradient algorithms,  in the more specific model-free settings of LQ controls and mean--variance portfolio selection, again leveraging the SA theory.

 %{ The above trilogy, however, does not address the important questions of convergence of the various algorithms devised and a regret analysis of the learned policies. Tang and Zhou \cite{Tang-Zhou}  tackle these questions and establish model-free convergence and sublinear regret bounds through a martingale-based $q$-learning framework: they alternate policy evaluation and policy improvement and combine backward stochastic differential equation (BSDE) estimates with stochastic-approximation theory to derive their results. Their analysis accommodates general function approximation but relies on relatively strong and rather complex assumptions on the stability of the associated stochastic approximation (SA) dynamics and the approximation architecture. By contrast, we bypass policy iteration and SA by establishing mutually coupled probabilistic representations of the value-function gradient and the $q$-function, which yield a directly implementable fixed-point loss independent of the unknown drift and running reward. Our approach requires the diffusion and terminal reward to be known, together with stronger data-access assumptions, but provides an {\it a posteriori} error bound for the learned solution. Meanwhile, \cite{HJZMV,HJZ,HZ} derive convergence and sublinear regret bounds for martingale-based policy gradient algorithms,  in the more specific model-free settings of LQ controls and mean--variance portfolio selection, again leveraging the SA theory.}

This paper studies stochastic controls with state dynamics governed by the It\^o diffusion type SDEs. The objective is to design RL algorithms and prove their convergence when the drift coefficients and the running reward functions are unknown. The specific approach we develop requires the diffusion coefficients and terminal reward functions (in the case of a finite time horizon) to be known and given, which is relatively reasonable because the former are easier to estimate (compared to the drifts) and the latter are often specified in many applications (e.g. a payoff of a stock option). %{ In contrast to \cite{Tang-Zhou}, we bypass policy iteration and SA by establishing mutually coupled probabilistic representations of the value-function gradient and the $q$-function, which yield a directly implementable fixed-point loss independent of the unknown drift and running reward. To the best of our knowledge, our paper is the first work on data-driven RL algorithms for exploratory HJB equations which contains both the convergence analysis and the numerical examples.}

Due to the unknown drifts and running rewards, we work in the exploratory realm of \cite{WZZ}, and aim to learn the optimal critic--actor pair $(u^*,\pi^*)$, namely the optimal value function and optimal (randomized) feedback policies, both functions of the time--state pair $(t,x)$. %In the classical model-based setting,
Theoretically, this pair can be approximated  through the so-called {\it policy iteration algorithm} (PIA). PIA is based on a very simple {\it coupled} relation between optimal critic and actor: the former is the conditional expectation of the total reward functional of the latter, and the latter is the maximizer (or soft maximizer) of the Hamiltonian which is a function of the gradient/Hessian of the former. This relation naturally implies an iterative scheme to approximate the two simultaneously. PIA type of algorithms along with their convergence have been extensively studied in the classical model-based setting (where exploration is not required), in both discrete-time and continuous-time; see the literature review in \cite{MWZ1}, the predecessor of this paper.

In the exploratory setting with entropy regularization, Wang and Zhou \cite{WZ}, for a mean--variance portfolio choice problem, show that a PIA converges in just two iterations. This is because the problem therein is essentially an LQ problem, for which the optimal Gibbs sampler reduces to Gaussian and thus highly tractable. See also Giegrich,
Reisinger, and Zhang \cite{GRZ2024} for related convergence results in the LQ setting. Huang, Wang,
and Zhou \cite{Huang2023} study a general infinite horizon
 model with control-independent diffusion coefficients and establish convergence of their PIA.
 Tran, Wang, and Zhang \cite{TWZ} consider a similar problem and investigate two cases depending on whether or not control enters into the diffusion term. In both cases, with different additional assumptions, they establish the PIA convergence. The methods of these two papers \cite{Huang2023,TWZ} are both partial differential equation (PDE)-based by analyzing the underlying exploratory HJB equations. By contrast, Ma, Wang and Zhang \cite{MWZ1} employ
 purely probabilistic method, primarily  probabilistic
representation formulae for the value function and its derivatives, to design PIAs and derive their convergence. A recent work by Huang, Yu, and Zhang \cite{HYZ2026} extends the
probabilistic method and proves the PIA convergence to a time-inconsistent problem with entropy regularization.

%The aforementioned papers \cite{Huang2023,TWZ,MWZ1} are all model-based, assuming oracle access to all the model primitives. They share two essential pitfalls. First, the policy iteration involves {\it derivatives} of the value function. While these papers have all proved that the value function iterates converge, numerically the derivatives of the iterates most likely will not converge to those of the true optimal value function. Second, their algorithms are not implementable when some of the coefficients are missing as in a model-free setting. See more detailed discussions on these two issues at the end of \S 2.1.

The aforementioned papers \cite{Huang2023,TWZ,MWZ1} are all
model-based, assuming oracle access to all the model primitives.
They share two essential pitfalls. First, policy iteration involves
{\it derivatives} of the value function. While these papers
establish convergence of the value function iterates, accurate
numerical approximation of the value function alone does not
guarantee accurate approximation of its derivatives. Second, the
aforementioned model-based PIAs are not directly implementable when
some of the coefficients are missing, as in a model-free setting.
See more detailed discussions on these two issues at the end
of \S~2.1.

This paper endeavors to overcome the two pitfalls, in the setting of unknown drifts and running rewards, by stepping up the probabilistic analysis of \cite{MWZ1}. The main part of the paper assumes that the diffusion coefficient does not depend on control, as in \cite{MWZ1}, in which case one needs only to consider the gradient (but not Hessian) of the value function in the analysis. Instead of the critic--actor pair $(u^*,\pi^*)$, we consider another pair $(v^*,w^*)$, where $v^*$ is the gradient of $u^*$ (in the spatial variable) and $w^*$ the Gibbs exponent in the expression of $\pi^*$ (or essentially the $q$-function). Delicate probabilistic representation and analysis reveal that this new pair also satisfies a coupled relation that happens to be {\it independent} of the drift and running reward functions. This in turn leads to an error analysis, as well as algorithms with theoretical guarantee on the convergence of both the value function and its gradient.\footnote{We note that a related difficulty arises when estimating the spatial gradient of a conditioning function; see Guo, Tang, and Xu \cite{GTX2026}.} The algorithms we develop are data-driven, and we define precisely what ``data" and ``data-driven" mean; see Remark \ref{rem-implementability} and Assumption \ref{data} for details.

As mentioned, the main results of the paper rely on the control-independence of the diffusion coefficients. For the very special case when the state space is one-dimensional, we extend our analysis and algorithm to include the case when control enters into diffusions in infinite time horizon. {In this  case  the (scalar) second-order spatial derivative can be algebraically expressed in terms of the  first-order derivative via the HJB equation, which helps to circumvent the main difficulty on analyzing the Hessian. For general fully nonlinear, multi-dimensional HJB equations, the problem remains open even with known $\si$, because one would need to estimate the full Hessian matrix, which calls for entirely new ideas.}

% and our special setting here is mainly to facilitate this analysis.

%This calls for an analysis on the Hessian.
%Despite a very restrictive setting, to the best of our knowledge this is the first data-driven algorithm with a rigorous convergence analysis for continuous-time {\it nonlinear} stochastic controls with unknown drifts and known yet control-{\it dependent} diffusions.

% We will also present a few numerical examples to demonstrate the efficiency and scalability of our learning algorithms.
% %In particular, we provide an example to compare with a now classical benchmark of Han, Jentzen and E \cite{E2018} for solving a class of high-dimensional PDEs based on deep BSDEs.
% { The result shows that our algorithm achieves a fairly good and stable performance %in terms of accuracy of the learned  solution to the exploratory HJB equation
% up to the state space dimension of 100 (which reaches the ceiling of our hardware capacity).}
% %, even though \cite{E2018} is model-based and ours is model-free.}

{ We will also present results of various numerical experiments to demonstrate the efficiency, flexibility, and scalability of our learning algorithms, including the special volatility control case. Our schemes accurately approximate both the optimal value function and the optimal policy, with reasonably large $T$. Notably, with GPU parallelization, our algorithm maintains stable and good performance in high-dimensional settings up to $d=100$, which reaches the memory capacity limit of our local hardware. Incidentally, we show (see Example \ref{EX-PIA}) that the direct implementation of PIA can indeed diverge if one simply uses auto-differentiation.}

Note that the exploratory HJB equations we are solving numerically
%This paper also provides a probabilistic numerical method to solve the exploratory HJB equations, which
constitute  a class of potentially high-dimensional, nonlinear parabolic or elliptic PDEs with certain unknown coefficients.Such PDEs with missing coefficients are termed {\it black-box} PDEs in a concurrent paper by
Jia et al. \cite{JOPZ}, which contains a comprehensive literature review and motivations of studying black-box PDEs and/or dealing with the curse of dimensionality. That paper studies a fully nonlinear (i.e. up to the Hessian) parabolic PDE whose coefficients are all unknown and develops model-free, data-driven algorithms.  Its key idea is to represent the gradient and Hessian via  the so-called zeroth-order derivative  estimators derived from perturbed Monte Carlo trajectories, which is fundamentally different from the representation in our paper. The PDE studied in \cite{JOPZ} consists of a linear operator which is the generator of a diffusion process plus a nonhomogeneous source term which is fully nonlinear up to the Hessian. Although it includes in form the exploratory  HJB
equation as a special case, the  assumption therein that the value of the nonlinear source term is known for
each given input excludes our setting. On the other hand, our equation clearly does not cover
the one in  \cite{JOPZ}. Therefore, the two papers are mutually exclusive and complementary to each other.

{ This paper employs a fundamentally different approach from \cite{Tang-Zhou} that was mentioned earlier. Specifically, we bypass policy iteration and SA used in \cite{Tang-Zhou}, by establishing mutually coupled probabilistic representations of the value-function gradient and the $q$-function, which yield a directly implementable fixed-point loss independent of the unknown drift and running reward.
Finally, let us compare our paper with  an interesting recent paper Mou \cite{Mou} that also studies convergence of continuous-time RL algorithms.
Although both papers are concerned with model-free control of continuous-time diffusions, the settings, aims and approaches differ substantially. Our paper works within the entropy-regularized exploratory-control framework and develops probability-based, on-policy algorithms that learn the value function, its derivatives, and the optimal randomized policy, with emphasis on convergence and scalable computation. By contrast, \cite{Mou} considers infinite-horizon off-policy learning from discretely observed offline data and exploits uniform ellipticity to design a so-called Sobolev-prox fitted Q-learning method with general function approximation and non-asymptotic statistical guarantees. In particular, it employs an $H^1$-norm based function approximation (instead of $\dbL^2$-approximation) for the value function $u^*$. This is a very interesting idea, and the analysis involved is deep. However, \cite{Mou} mainly focuses on the theoretical analysis. The efficiency of the proposed algorithm remains to be investigated and there are no numerical experiments reported in the paper. In sum, our paper is primarily a probabilistic stochastic-control and numerical-algorithmic contribution, whereas \cite{Mou}  is more centrally a statistical-learning theory grounded in elliptic PDE and functional-analytic structure. To the best of our knowledge, our paper is the first work on data-driven RL algorithms for general exploratory HJB equations with both a convergence analysis and substantial numerical experiments.

}

The rest of the paper is organized as follows. In \S\ref{sect-idea} we introduce the problem and highlight our main ideas. In \S\ref{sect-small} we carry out an analysis necessary for designing the main algorithm and its convergence when time duration is small. The analysis is extended to the general time horizon in \S\ref{sect-multi}. In \S\ref{sect-numerical} we present the algorithms along with numerical examples. In \S\ref{sect-volatility} we study a special case with control-dependent diffusion coefficients and one-dimensional state space, with a numerical example. Finally, \S\ref{conclusion} concludes. Some of the proofs are placed in Appendix.

%\bs

\section{Problem Formulation}
\label{sect-idea}
\setcounter{equation}{0}

\no{\bf Notations.} Here we list a few notations that will be used frequently in the paper.

\begin{itemize}
\item For $x, \tilde x\in \dbR^n$, $x\cd \tilde x $ denotes the inner product, and $|x|$ the Euclidean norm.

\item For $M, \tilde M\in \dbR^{m\times n}$, $M^\top$ denotes the transpose of $M$, and $M:\tilde M := \tr(M \tilde M^\top)$. Moreover, $I_d$ denotes the $d\times d$ identity matrix.

\item For a generic topological space $\bE$, $C^0(\bE; \dbR^n)$ denotes the space of continuous functions $\f: \bE\to \dbR^n$, and $C^0_b(\bE; \dbR^n)$ the subspace of functions with finite uniform norms:
\bea
\label{uniform}
\|\f\|_0:= \sup_{\bE} |\f|.
\eea

%\item For a generic Euclidean space $\bE$, $C^k(\bE; \dbR^n)$ denotes the space of $k$-th differentiable functions $\f: \bE\to \dbR^n$, and $C^{m,k}([0, T]\times \bE; \dbR^n)$ the space of functions which are $m$-th differentiable in $t\in [0, T]$ and $k$-th differentiable in $x\in \bE$.

\item For a measurable set $A$ in some Euclidean space, $\cP_0(A)$ denotes the space of probability density functions $\pi: A\to [0, \infty)$, namely $\pi(a)\geq0\;\;\mbox{a.e.}$ and $\int_A \pi(a) da=1$.  Moreover, $\cH(\pi)$ denotes the Shannon entropy of $\pi$:
\bea
\label{Shannon}
\cH(\pi) := - \int_A \pi(a) \ln \pi(a) da,\;\;\pi \in \cP_0(A).
\eea

\item For a generic topological space $\bE$, a measurable set $A$ in some Euclidean space, and a measurable function $\f: \bE\times A\to \dbR^n$, denote
\bea
\label{fpi}
\tilde \f(x, \pi) := \int_A \f(x, a) \pi(a) da,\q x\in \bE, \pi \in \cP_0(A).
\eea
\end{itemize}
%{Throughout this paper, we fix %a finite time horizon $[0, T]$, and %assume that all randomness comes from
%a  filtered probability space $(\O, \cF, \hP; \hF)$ on which is defined
%a standard $d$-dimensional Brownian motion $B$. We further assume that $\hF=\hF^B$ and all control actions take values in a given domain $A$ in some Euclidean space with
%% for control values. We assume $\dbF=\dbF^B$ and $A$ has
%finite volume:}
%$$
%0<|A|<\infty.
%$$
%The main problem considered in this paper is over a finite time horizon $[0, T]$, together with a  filtered probability space $(\O, \cF, \hP; \hF)$, where $\dbF=\dbF^B$ and $B$ is  a standard $d$-dimensional Brownian motion $B$. We also fix a domain $A$ in some Euclidean space for control values, assumed to have a finite volume:
%$$
%0<|A|<\infty.
%$$

\subsection{Entropy-regularized stochastic control}

Fix a finite time horizon $[0, T]$ and %assume that all randomness comes from
a filtered probability space $(\O, \cF, \hP; \hF)$ on which is defined
a standard $d$-dimensional Brownian motion $B$, where $\hF$ is the natural filtration generated by $B$. We are also given an action space $A$, which is a measurable set  in some Euclidean space with
a finite volume:
$$
0<|A|<\infty.
$$

Our primal interest is the following classical, {\it target} control problem: for $(t, x)\in [0, T]\times \dbR^d$,
\bea
\label{value0}
\left.\ba{c}
\dis X^{t,x,\a}_s = x+\int_t^s b(l,X^{t,x, \a}_l, \a_l)dl + \int_t^s \si(l,X^{t,x, \a}_l)dB_l,\;\;s\in [t,T];\\
\dis u^*_0(t, x) := \sup_{\a} \dbE\Big[g(X^{t,x,\a }_T) + \int_t^T r(s,X^{t,x,\a }_s, \a_s) ds\Big].
\ea\right.
\eea
{In the above, an {\it admissible control} $\a$ is an $A$-valued and $\dbF$-progressively measurable process, and $b, \si, r, g$ are appropriate coefficients taking values in $\dbR^d$, $\dbR^{d\times d}$, $\dbR$, and $\dbR$, respectively.\footnote{In this paper we assume that the state and Brownian motion have the same dimension for ease of exposition. When they have different dimensions, the analysis is essentially the same except pseudo matrix inverse will need to be involved in certain places.} For the most part of the paper, we assume that $\si$ is control-independent as in \eqref{value0}, but a special case of controlled $\si$ will be investigated in \S\ref{sect-volatility}.}

%{We are interested in designing efficient numerical algorithms for solving the problem (\ref{value0}) along with the associated HJB equation with model uncertainty, in the sense that some or all of the coefficients are unknown.}

The key feature considered in this paper is that the drift term $b$ and the running reward function $r$ are unknown. Therefore, classical approaches such as dynamic programming (via HJB equations) and maximum principle (via Hamiltonian and FBSDEs) fail to solve  the above problem.  We resort to the continuous-time reinforcement learning (RL) paradigm introduced by \cite{WZZ} and consider the
following {\it exploratory},  entropy-regularized problem:
\bea
\label{valueu}
\left. \ba{lll}
\dis X^{t,x,\pi}_s = x+\int_t^s \tilde b(l,X^{t,x, \pi}_l, \pi(l,X^{t,x,\pi}_l))dl + \int_t^s \si(l,X^{t,x, \pi}_l)dB_l,\;\;s\in [t,T];\ms\\
%\dis d X^{\pi}_s = \tilde b(s,X^{\pi}_s, \pi(s,X^{\pi}_s))ds + \si(s,X^{\pi}_s) dB_s;\ms\\
\dis J(t,x; \pi)\neg:=\neg \hE\Big[g(X^{t,x,\pi}_T)\neg +\neg\int_t^T\neg\neg \big[\tilde r(s,X^{t,x,\pi}_s, \pi(s, X^{t,x,\pi}_s)) + \l \cH(\pi(s, X^{t,x,\pi}_s))\big] ds\Big],\ms\\
\dis u^*(t, x) := u^*_\l(t, x) :=  \sup_{\pi\in \cA_T} J(t,x; \pi),\q  (t, x)\in [0, T]\times \dbR^d.
\ea\right.
\eea
Here,  $\cA_T$ denotes the space of feedback type relaxed controls/policies $\pi: [0, T]\times \dbR^d\to \cP_0(A)$, $\tilde{b}$ and $\tilde{r}$ are the ``convexifications" of $b$ and $r$ respectively defined in (\ref{fpi}), and $\l>0$  is an exogenous ``temperature parameter" capturing a balance between exploitation and exploration essential in the RL approach. It is shown in \cite{WZZ} that the value function $u^*$  satisfies the  {\it exploratory HJB equation}
\bea
\label{HJBu}
\left.\ba{lll}
\dis u^*_t+ \frac12 [\si\si^\top](t, x) : u^*_{xx}  + H(t, x, u^*_x) =0,\q u^*(T,x) = g(x),\;\;\mbox{ where }\ms\\
\dis  H(t, x, z) := \sup_{\pi\in \cP_0(A)} \big[~\tilde b(t, x, \pi) \cd z + \tilde r(t, x, \pi) + \l \cH(\pi)\big],
\ea\right.
\eea
and
the optimal relaxed control $\pi^*$ takes the Gibbs form
\bea
\label{pi*}
\left.\ba{c}
\dis \pi^*(t, x, a):=\G(t, x, u^*_x(t, x), a), \q (t, x, a)\in [0, T]\times \hR^d\times A,\q\mbox{where}\ms\\
\dis  \G(t, x, z, a) := \frac{\g(t,x,z,a)}{\int_A \g(t,x,z,a') d a'}, \mbox{ with }\g(t,x,z,a):= \exp \Big(\frac{1}{\lambda}[b(t,x, a) \cdot z+r(t, x, a)]\Big).
\ea\right.
\eea
With the above notation, the Hamiltonian $H$ in (\ref{HJBu}) can be rewritten as
 \bea
 \label{H}
  H(t,x,z)=\lambda \ln \Big(\int_A \g(t,x,z,a) d a\Big).
  \eea

The above results show that one should use a Gibbs sampler in general to generate trial-and-error
strategies to explore the environment when certain model coefficients are unknown.
\cite{TZZ} further shows that the exploratory control problem converges to the original  target control problem when the temperature parameter $\lambda\rightarrow 0$.
In other words, the target problem (\ref{value0}) can be solved via the exploratory problem (\ref{valueu}) in the RL setting. The goal of this paper is therefore
to design efficient numerical algorithms with theoretical guarantees for computing the above $(u^*, \pi^*)$ without knowing $b$ and $r$.  Note that besides numerically solving both the target and exploratory control problems, solving the exploratory HJB \eqref{HJBu}, which is a nonlinear parabolic PDE (and high dimensional in many applications),  is interesting in its own right from a PDE perspective.\footnote{\cite{TZZ} establishes the well-posedness and regularity of the viscosity solution to the exploratory HJB equation, but does not address the problem of numerically  computing the solution.}

To achieve the goal of this sort, the following {\it policy iteration algorithm}  (PIA) has been popular in the literature. Given appropriate initialization $u^0$, define $(\pi^n, u^n)$, $n\ge 1$, recursively:
\bea
\label{PIA}
\pi^{n}(t, x, a):=  \Gamma\left(t, x,  u_x^{n-1}(t, x), a\right),\q u^{n}(t, x):=J(t, x; \pi^n).
\eea
In particular, under some  technical conditions, it is shown in \cite[Theorem 2.4]{MWZ1} that the PIA converges with a super-exponential rate: for some $0<\eta<1$,
\bea
\label{PIAconv}
\|u^n-u^*\|_0 +  \|\pi^n-\pi^*\|_0 \le C \eta^{2^n}.
%\|u^n_x - u^*_x\|_0 +
\eea

Despite the very strong theoretical convergence, the above algorithm has two fundamental drawbacks from the {\it numerical} and {\it learning} perspectives, which we now explain. First, the iterate $\pi^n$ in \reff{PIA} depends on $u^{n-1}_x$, whereas  the above iterative scheme only returns the approximated value of $u^n$, say $\hat u^n$. It is well known that the convergence of a function (under, say, the uniform norm \reff{uniform}), does not necessarily lead to the convergence of its derivatives. Therefore, even if $\hat u^{n-1}$ approximates $u^{n-1}$ well,   $\hat u_x^{n-1}$ may not be a desired approximation of $u^{n-1}_x$.\footnote{For example, when using deep neural networks, derivatives are typically computed via the so-called
automatic differentiation (auto-diff) which are the exact derivatives of the network representation that approximates a target function. } Consequently $\hat \pi^n(t,x,a) := \Gamma\left(t, x,  \hat u_x^{n-1}(t, x), a\right)$ may not approximate $\pi^n$ satisfactorily, with errors propagating into  subsequent iterates. Thus, in {\it actual} implementation, the above PIA algorithm may not converge at all. Indeed, in \S\ref{sect-numerical} we will present a numerical example (Example \ref{EX-PIA}) showing that a simple minded PIA actually diverges.\footnote{The proof of \cite[Theorem 2.4]{MWZ1} also shows that $\|u^n_x - u^*_x\|_0$ has the same error bound of $C \eta^{2^n}$. However, it is a {\it theoretical}  result assuming that in each iteration the derivative of the {\it oracle}, instead of an approximate (e.g. auto-diff), solution of the PDE involved is used. The numerical and algorithmic aspects of the problem are not studied in \cite{MWZ1}.}
%\footnote{\label{parameter}We shall note that the PIA may work when the model uncertainty lies only in some low dimensional parameter $\th$, e.g. $b(t,x,a)= b_0(\th; t,x,a)$, where $b_0$ is known, but it depends on some unknown parameter $\th$. In this case, a good approximation $\hat u^{n-1}$ could possibly lead to a good approximation $\hat \th^{n-1}$ of $\th$, which implies further $\hat u^{n-1}_x$ is a desired approximation of $u^{n-1}_x$. However, in this case quite often it is more efficient to estimate $\th$ directly, by using the data learned by trial, and then to apply the model based numerical methods to compute the $u^*_0$ in \reff{value0}, even without the need of entropy regularization.  }

The more serious issue arises from  the unknown model parameters:
the Gibbs function $\g$ in (\ref{pi*}) depends on the coefficients $b$ and $r$; so the algorithm is not implementable when they are unknown. Since the main purpose of the exploratory framework is to deal with such   models with unknown parameters, we will make this a focal  point in developing our new algorithms in this paper.

\subsection{Plan of attack}
%\label{SS2.1}
%\setcounter{equation}{0}

Bearing the aforementioned two main issues in mind, in this paper we propose new implementable algorithms and analyze their convergence. In this subsection we highlight the main ideas for reader's convenience.

To overcome the first issue, we approximate $u^*_x$ {\it directly} by utilizing the Bismut--Elworthy--Li representation formula \cite{Bismut,EL}. % (which has also been employed in \cite{MWZ1}).
Specifically, denote
\bea
\label{vw*}
v^*(t,x):=u^*_x(t,x), \q w^*(t, x,a):= v^*(t,x)\cdot b(t,x,a) +r(t,x,a).
\eea
Then, by \reff{pi*} and \reff{H}, we may rewrite the exploratory  HJB equation (\ref{HJBu}) as
\bea
\label{HJBw}
%\left\\ba{lll}
\dis u^*_t+ \frac12 [\si\si^\top](t, x) : u^*_{xx}  + \l \ln\int_A e^{-\frac1{\l} w^*(t, x, a)}da=0; \q  u^*(T,x) = g(x).
%\dis  w(t, x, a) := \pa_xu(t,x)\cd b(t, x, \pi) + r(t, x, a), \q (t,x,a)\in[0,T]\times\hR^d\times A.
%\ea\right.
\eea
It follows from the Feynman--Kac and Bismut--Elworthy--Li  formulae  that we have the following probabilistic presentations of $u^*$ and $v^*$:
\bea
%\left.\ba{lll}
\label{u*rep}
&&\!\!\!\!\!\!\!\!\!\!\!\!\!\!\!\!\!\! \dis u^*(t,x)=\hE\Big\{g(\cX_T^{t,x}) + \int_t^T \lambda\ln\Big[\int_A \exp \big(\frac{1}{\lambda}w^*(s,\cX^{t,x}_s,a)\big)da\Big]ds\Big\}; \ms\\
\label{v*rep}
&&\!\!\!\!\!\!\!\!\!\!\!\!\!\!\!\!\!\! \dis v^*(t,x)=\hE\Big\{ (\nabla \cX_T^{t,x})^\top g_x(\cX_T^{t,x}) + \int_t^T \lambda\ln\Big[\int_A \exp \big(\frac{1}{\lambda}w^*(s,\cX^{t,x}_s,a)\big)da\Big]N_s^{t,x}ds\Big\}.
%\ea\right.
\eea
In the above, $\cX^{t,x}$ denotes the dynamics of the following {\it reference state}:
\bea
\label{cXtx}
\cX^{t,x}_s = x + \int_t^s \si(l, \cX^{t,x}_l) dB_l, \q s\in[t,T],
\eea
and $N^{t,x}$ is the Bismut--Elworthy--Li representation kernel defined by
\bea
\label{DNtx1}
\left.\ba{lll}
\qq \dis N^{t,x}_s:=\frac{1}{s-t}\int_t^s ( \si^{-1}(l, \cX^{t,x}_l)\td \cX^{t,x}_l)^\top dB_l, \q s\in[t, T],
\ea\right.
\eea
where $\td \cX^{t,x}$ is the {\it variational process} of $\cX^{t,x}$, which satisfies the following linear SDE:
%be the matrix with $r$-th column $\pa_{x_r}\cX^{t,x}$, using Einstein summation for the repeated indices, we have:
\bea
\label{DcXtx}
\qq \left.\ba{lll}
\dis \td \cX^{t,x}_s = I_d + \sum_i\int_t^s\si_x^i(l, \cX^{t,x}_l) \td \cX^{t,x}_ldB^i_l, \q s\in[t, T],
\ea\right.
\eea
with $\si^i$  denoting the $i$-th column of $\si$ and $B^i$ the $i$-th component of $B$. It is crucial to note that the representation system (\ref{u*rep})--(\ref{DcXtx}) depends  on the functional forms of $\si$ and $g$, but  {\it not} those of $b$ and $r$.

The remaining task is to find an effective and implementable way to learn the function $w^*$, again without involving $b$ and $r$. { Our idea is based on the following simple yet crucial observation: for $\D t>0$ small and $a\in A$, applying It\^o's formula we get
\bea\label{widea}
\dbE\big[u^*(t,X_{t+\D t}^{t,x,a}) - u^*(t,\cX^{t,x}_{t+\D t})\big]  &\equiv & \dbE\big[u^*(t,X_{t+\D t}^{t,x,a}) - u^*(t,x)\big] -\dbE\big[ u^*(t,\cX^{t,x}_{t+\D t})- u^*(t,x)\big]\nonumber\\
%&\approx& \dbE\big[u^*_x(t,x)\cd (X_{t+\D t}^{t,x,a}-\cX^{t,x}_{t+\D t})\big] \nonumber\\
&\approx& v^*(t,x)\cd b(t,x,a)\D t,
\eea
where $X^{t,x,a}$ is the state in \reff{value0} under a constant control $\a\equiv a$.
%With the help of \reff{widea}, we are able to approximate the key element inside the Hamiltonian without using $b$ directly,
%where $\D t>0$ is a small time step. Consequently,  using Newton-Leibniz formula
Noting the general formula
\bea
\label{uv} u(t,x_2)-u(t,x_1)=\cI\big(v; t, x_1,  x_2\big) := \int_0^1 v\big(t, (1-\th) x_1 + \th x_2\big) d\th~\! \cd (x_2 - x_1\big)\eea
where $v(t,x)=u_x(t,x)$,
 we have the following estimate of  $w^*$:
 %\footnote{{ The following \eqref{w*rep} is the high-dimensional extension of the following more evident approximation when the state processes are one-dimensional:
%$w^*(t, x,a)
%\approx {1\over \D t} \dbE\Big[\int_{\cX^{t,x}_{t+\D t}}^{X^{t,x,a}_{t+\D t}}  v^*(t, x') dx'\Big] + r(t,x, a).$}}
\bea
\label{w*rep}
\left.\ba{c}
\dis w^*(t, x,a)
%&:=&v^*(t,x)\cdot b(t,x,a) +r(t,x,a)\nonumber \\
%&\approx&\frac{1}{\D t}\bigl[u^*(t,X_{t+\D t}^{t,x,a})-u^*(t,x)+u^*(t,x)-u^*(t,\cX^{t,x}_{t+\D t})+r(t,x,a)\D t\bigr]\nonumber\\
%&\approx&
%\approx {1\over \D t} \dbE\Big[\int_x^{X^{t,x,a}_{t+\D t}}  v^*(t, x') dx'\Big] - {1\over \D t} \dbE\Big[\int_x^{\cX^{t,x}_{t+\D t}}  v^*(t, x') dx'\Big]+ r(t,x, a).
\approx {1\over \D t}{\dbE\Big[\cI\big(v^*; t, \cX^{t,x}_{t+\D t},  X^{t,x,a}_{t+\D t}\big)\Big]}+ r(t,x, a). %\ms\\
%\dis \mbox{where}\q  { \cI\big(v; t, x_1,  x_2\big) := \int_0^1 v\big(t, (1-\th) x_1 + \th x_2\big) d\th~\! \cd (x_2 - x_1\big).}
\ea\right.
\eea}

{Our algorithms will be based on solving the approximate ``fixed point" $(v^*, w^*)$ in \reff{v*rep} and \reff{w*rep}, which we develop in the following sections. Before we proceed, we would like to reiterate about  what problem parameters are known and what are unknown, as they will be key in determining the  implementability of our algorithm.
%and to demonstrate the implementability of our algorithm to be designed,
We do this in the following remarks.

\begin{rem}
\label{rem-implementability}
{\rm (i) Recall that the main feature of this paper is that the functional forms of $b$ and $r$ are unknown, and we aim to  devise {\it data-driven} solutions. We now make the notions of ``data" and ``data-driven" precise in our setting.
Even though $b$ is unknown, we require
the state process $X^{t,x,a}$ under any given control $a$ is {\it observable} (i.e. the process constitutes {\it data}). This is a natural requirement. For example, in investment the state process is the wealth process of a portfolio (the control). Given a portfolio policy, we can certainly observe our corresponding wealth process even if we do not know anything about the price dynamics of the stocks involved in our portfolio.
%can be learned through trials/experiments. We remark that we will need the availability of the data $X^{t,x,a}_{t+\D t}$ for any given $(t,x, a)$.
Likewise, while we do not know the running reward function $r$, we assume we will get the function {\it value} $r(t,x,a)$, a so-called {\it reward
%} (or {\it reinforcement}) {\it
signal}, whenever a control $a$ is applied at $(t,x)$. A data-driven RL algorithm
makes use of these observed data along with other available or simulated data to learn final solutions. See e.g. \cite{JZ1,JZ2,JZ3} for more discussions on this issue.

\ss
 (ii) There is a more subtle point regarding available data. For any given $(t,x, a)$, in calculating $v^*$ based on \reff{v*rep} we need to compute $w^*(s, \cX^{t,x}_s, a)$ for $s\in[t,T]$, which in view of \reff{w*rep} requires the data $X^{s, \cX^{t,x}_s, a}_{s+\D t}$
 and $\cX^{s, \cX^{t,x}_s}_{s+\D t}$ for all $s$ and $a$. But since $\cX^{t,x}_s$ is simulated, its trajectory  could spread over the whole space. This in turn demands us to observe the state data process starting from $(s,y)$ for {\it any} given $y$, which is not necessarily on the {\it original} state trajectory starting from $(t,x,a)$ (i.e. in general $y\neq X^{t, x, a}_{s}$). We will formalize this type of data  availability as an assumption; see Assumption \ref{data}-(i) along with discussions on this assumption in Remark \ref{rem-data} in the next section.
 %This may unfortunately increase the training cost of our algorithm significantly, especially when it is not cheap to obtain the data $X$.}

 \ss
(iii) { In this paper we assume that the function $\si$ is known due to two reasons: we need to simulate the reference state process \eqref{cXtx} and  need access to $\si_x$ in view of  \reff{DcXtx}.  This assumption is relatively benign because it is well understood that,  comparing to the mean $b$, the volatility $\si$ is much easier to estimate to very high accuracy.  However, we do note that a good estimate of $\si$ does not always leads to a good estimate of $\si_x$.} Given $\si$, we can {\it simulate}  $\cX$ and $\td \cX$ at low costs.  Similarly, we assume the terminal reward function $g$ (and hence $g_x$) is known, which is reasonable in many applications where $g$ is a known utility or payoff function (e.g. the payoff function of a financial derivative). %We will formalize the accessibility to the data and functions in Assumption \ref{data} in the next section.

%\qed
%\end{rem}
%}
%
%\begin{rem}
%\label{rem-implementability2}
%{ \rm In addition to Remark \ref{rem-implementability}, we should also note that

\ss
 (iv) The introduction of the function $w^*$ is central in our approach. Equation (\ref{v*rep}) by itself does not solve $v^*$, but coupled with (\ref{w*rep}) the pair $(v^*,w^*)$ can be solved {\it simultaneously}. Moreover, $u^*$ can also be determined by $w^*$. Now, $w^*$ differs from the Hamiltonian of the {\it target}  problem by a control-{\it independent} term
 ${1\over 2} [\si\si^\top](t, x) : u^*_{xx}(t,x)$. In view of the policy improvement theorem
 (\cite[Theorem 2]{JZ3}) and the definition of the $q$-function (\cite[Definition 4]{JZ3}), $w^*$ is essentially the $q$-function in our particular setting where $\sigma$ is independent of $a$. The general $q$-learning theory developed in \cite{JZ3} employs martingale conditions to design algorithms to learn the $q$-function without giving convergence results. \cite{Tang-Zhou} investigates the convergence and regret of $q$-learning based on BSDE and stochastic approximation, under technically complex and strong assumptions. By contrast, our paper takes a very different (and delicate) approach via representing the two key functions with each other: the gradient of the optimal value function and the (essentially) $q$-function, at the cost of having to assume $\sigma$ and $g$ to be known.

 \ss
 (v)
In theory $\cX$ and $X$ can be  driven by different  Brownian motions, although we use the same notation $B$ in \eqref{value0} and \eqref{cXtx}. Indeed, in the numerical examples in  \S\ref{sect-numerical} below, we will use independent Brownian motions to simulate $\cX$ and generate the environmental data $X$ respectively. However, the right side of \reff{w*rep} involves the difference of expectations, which are law invariant.  Thus
% there is no problem to
employing  the {\it same} Brownian motion $B$ does not cause any essential differences in theoretical  analysis. }
\qed
\end{rem}

\section{Small Time Horizons}
\label{sect-small}
\setcounter{equation}{0}

In this section we carry out an analysis following the aforementioned idea, under the assumption that the time horizon $T$ is sufficiently small. The case of a general time horizon will be studied in the next section.

In the remainder of the paper, we will impose the following {\it Standing Assumptions} for the finite horizon setting, without stating explicitly in the results.  The first one is about the regularity of the problem coefficients (even though some of them are unknown).

\begin{assum}\label{assum-standing}
There exist constants $L_0, L_1, L_*>0$ such that the following hold true:

(i) $b, \si, r$ are bounded by $L_0$, and $\si$ is uniformly non-degenerate: $\si \si^\top \ge {1\over L_0} I_{d}$.

\ss
(ii) $b, r$ are measurable in $a$, and $b, \si, r$ are uniformly Lipschitz continuous in $x$ and uniformly ${1\over 2}$-H\"{o}lder continuous in $t$, with a Lipschitz/H\"{o}lder constant $L_1$.

\ss
(iii) The PDE \reff{HJBu} has a unique classical solution $u^*$ satisfying
\bea
\label{L*}
\|u^*_x\|_0+\|u^*_{xx}\|_0 + \|u^*_{xxx}\|_0 \le L_*.
\eea
\end{assum}
Note that assumptions on $g=u^*(T,\cd)$ are implied by (iii).
%Clearly, under Assumption \ref{assum-standing} the SDE  \reff{controlu} has a unique weak solution  $X^{\pi}$. Moreover, the regularity of the coefficients $b$, $\si$, $r$, and $g$ also guarantees that the classical solution $u^*$ exists, which renders that $(v^*, w^*)\in \sC_\infty$.
Throughout  we will use a generic constant $C$ (i.e. its value may change from line to line), which depends only on $d$, $m$,  $\l$, $|A|$, and $L_0, L_1$, but not on $L_*$ or $T$. When the latter dependence occurs, we use the notation $C_{L_*, T}$.
% and $C_0$, but is independent of $T$ and $g$.

\begin{rem}
\label{rem-u*}
{\rm Assumption \ref{assum-standing} ensures that the state equation \reff{value0} has a unique strong solution for any given control $\alpha$ as well as a unique weak solution for any constant control $a$. On the other hand, since the PDE \reff{HJBu} is semi-linear, it is standard (see e.g. \cite[Chapter 7]{F}) to find sufficient conditions to ensure Assumption \ref{assum-standing}-(iii) by increasing the regularity of $b, \si, r$. However, because the key estimate for the convergence of our new algorithm, based on Lemma \ref{lem-contraction} below, relies only on $L_0, L_1$ but not on $L_*$,  we prefer not to impose further regularity assumptions on $b, \si, r$  so that we can focus on $L_0, L_1$.
Also, for our theoretical analysis, we require only the measurability in $a$. However, for the numerical examples in  \S\ref{sect-numerical}, we need to approximate integrations in $a$ by Riemann sums, in which case certain regularity in $a$ will further be needed.}
\qed
\end{rem}
The next assumption is on the data accessibility/availability from the environment.
\begin{assum}\label{data}
(i) The environment  returns a trajectory $\{X^{t,x,a}_s:t\leq s\leq T\}$, for each query with input
$(t, x,a)\in [0, T]\times \dbR^d\times A$,  which is the (unique) weak solution to \reff{value0} under the constant control $a$.

\ss
(ii) The environment  returns a  value $r(t,x,a)$ for each $(t, x,a)\in [0, T]\times \dbR^d\times A$.

\ss
(iii) The functional forms of $\sigma$ and $g$ are known and given.
\end{assum}

\begin{rem}
\label{rem-data}
{\rm From a practical viewpoint it is a rather strong assumption to require
the environment to return the state trajectory for {\it any} $(t,x,a)$. Imagine one controls a state process with a constant control $a$ starting from initial time 0 and initial state $x_0$.
She will naturally observe the resulting data trajectory {\it as she goes}, even if she may not know the drift $b$.  At time $t$, the state is at $X^{0,x_0,a}_t$ so she will continue to see the trajectory from $(t,X^{0,x_0,a}_t,a)$. However, Assumption \ref{data}-(i) requires her, at time $t$,  to also see the data process starting from an {\it off-trajectory} state $x$, which may not be possible in general.

This being said, there are (important) cases where this assumption is satisfied. Consider the wealth equation in stock investment
\[ d X_t=[r X_t+(\mu-r)a]dt+\sigma a dB_t,\]
where $r$ is the risk-free rate (which is assumed to be known as it is typically the bank saving rate or the bond yiled), $\mu$ and $\sigma$ are the return rate (which is hard to estimate and hence unknown) and volatility rate respectively of a stock, $a$ is the portfolio (the amount allocated to the stock), and $X_t$ (which is observable) is the total wealth at time $t$. Here we assume there is only one stock but the following analysis extends trivially to multiple stocks. Because $\mu$ is unknown, the state equation above
 has an unknown drift.
Denote by $X^{0,x,a}$ the solution to the equation with $X_0=x$. Then it is immediate that
\[ X^{0,y,a}_t=X^{0,x,a}_t+e^{rt}(y-x),\;\;\;y\neq x.\]
In other words, the data process starting from a {\it different} state $y$ can be {\it inferred} from the data process starting from the given $x$ that can be observed. In this case, Assumption \ref{data}-(i) holds. So the results in this paper can be applied to a broad class of financial
portfolio selection problems.

The assumption  is also (approximately) satisfied in several other practical settings. For example, a rideshare platform can monitor vehicle trajectories starting from various locations within a region, assuming a sufficiently high vehicle density.
A stock trading app  with a huge number of users (such as Robinhood) can observe wealth trajectories of different accounts starting from various amounts.

Finally, if $b$ (along with $\sigma$) is known then Assumption \ref{data}-(i) becomes automatic because one can simulate (and therefore observe) the state process starting from any point.}
\qed
\end{rem}

Fix $K\ge 1$ and consider the uniform time discretization $\hT_K:=\{t_k\}_{k=0}^{K}$ of $[0, T]$, with $\D t := {T\over K}$ and $t_k = k \D t$, $k=0,\cds, K$. Denote $\D B_{t_{k+1}}=B_{t_{k+1}}-B_{t_{k}}$. For each $(t_k, x)$, we first discretize the reference state \reff{cXtx} and its variational process \reff{DcXtx}:
 \bea
\label{cXdiscretization}
\left.\ba{lll}
\dis \cX^{t_k, x, \D t}_{t_k}=x;\q \cX^{t_k, x, \D t}_{t_{i+1}}=\cX^{t_k, x, \D t}_{t_{i}}+\sigma(t_i,\cX^{t_k, x, \D t}_{t_i})\D B_{t_{i+1}};\\
\dis \td \cX^{t_k, x, \D t}_{t_k}= I_d;\q \td \cX_{t_{i+1}}^{t_k, x, \D t} =\td \cX^{t_k, x, \D t}_{t_{i}}+\sum_{j=1}^d\sigma^j_x(t_i, \cX^{t_k, x, \D t}_{t_i} )\td\cX^{t_k, x, \D t}_{t_{i}}\D B^j_{t_{i+1}},
\ea\right.
\eea
for $i=k, \cds, K-1$.  Moreover,  we define the discretized kernel process from \reff{DNtx1}
\bea
\label{NtkxDt}
N^{t_k,x,\D t}_{t_i}=\frac{1}{t_i-t_k}\sum_{j=k}^{i-1} ( \sigma^{-1}(t_j, \cX^{t_k,x,\D t}_{t_j})\td \cX^{t_k,x,\D t}_{t_j})^\top \D B_{t_{j+1}}, ~ i=k+1,\cds, K.
\eea
Note that, due to the singularity of $N^{t,x}_s$ at $s=t$, the above iteration starts with $i=k+1$.

The following estimates are standard; see e.g. \cite{ZhangBSDE}. %and we shall sketch a proof in Appendix.
%\begin{lem}
%\label{lem-Kerr}
%Under Assumption \ref{assum-standing}, for any $p\ge 2$,  there exists a constant $C_p>0$ such that, for any $(t_k, x)$,
%\bea
%\label{Kerr1}
%\left.\ba{c}
%\dis \dbE_{t_k, x}\Big[ \sup_{t_k\le s\le T} \big[|\cX_s -x|^p + |\td\cX_s|^p\big]\Big] \le C_p;\ms\\
%\dis\max_{k\le i\le K}  \dbE_{t_k,x}\Big[|\cX_{t_i} - \cX^{ \D t}_{t_i}|^p +  |\td\cX_{t_i} - \td\cX^{\D t}_{t_i}|^p\Big] \le C\D t^{p\over 2}.
%% \dis \dbE\big[|N^{t_k,x}_s|^p \big] \le {C_p\over (s-t_k)^{p\over 2}},\q \dbE\Big[|N^{t_k,x}_{t_i} - N^{t_k,x, \D t}_{t_i}|^2 \Big] \le {C\over i-k}.
%\ea\right.
%\eea
%\end{lem}
\begin{lem}
\label{lem-Kerr}
%Under Assumption \ref{assum-standing},  t
There exists a constant $C_T>0$, depending only on $d$,  $\l$,  $L_0$, $L_1$, and $T$ but not on $L_*$, such that for any $(t_k, x)$ and $s>t_k$, $i>k$,
\bea
\label{Kerr1}
\left.\ba{c}
\dis \dbE_{t_k, x} \big[|X^{a}_{s}-x|^4+|\cX_{s} -x|^4\big] \le C_T(s-t_k)^2,\q \dbE_{t_k, x} \big[ |\td\cX_s|^4 + |\td \cX^{\D t}_{t_i}|^4\big] \le C_T,\ms\\
\dis  \dbE_{t_k, x} \big[|N_{s}|^4\big]\le {C_T\over (s-t_k)^2},\q  \dbE_{t_k, x} \big[|N^{\D t}_{t_i}|^4\big]\le {C_T\over (t_i-t_k)^2},\ms\\
\dis \dbE_{t_k,x}\Big[|\cX_{t_i} - \cX^{ \D t}_{t_i}|^4 +  |\td\cX_{t_i} - \td\cX^{\D t}_{t_i}|^4\Big] \le C_T(\D t)^2,
% \dis \dbE_{t_k,x}\big[|N^{\D t}_{t_i}|^2 \big] \le {C_p\over (s-t_k)^{p\over 2}},\q \dbE\Big[|N^{t_k,x}_{t_i} - N^{t_k,x, \D t}_{t_i}|^2 \Big] \le {C\over i-k},
\ea\right.
\eea
where $X^{a}$ is the state process starting from $(t_k, x)$ under the constant control $a$.
Moreover, the above $C_T$ is increasing in $T$, and in particular $C_T \le C_{\bar T}$ if $T\le \bar T$.
\qed
\end{lem}
\no Here and henceforth, for notational simplicity we take the convention that by using $\hE_{t_k, x}$ we omit the superscripts $^{t_k,x}$ for the processes under the expectation. For example, in \reff{Kerr1},
\bea
\label{convention}
\cX = \cX^{t_k,x},\q \td \cX^{\D t} = \td \cX^{t_k, x,\D t}.
\eea
Next, inspired by the relations (\ref{v*rep}) and (\ref{w*rep}) we introduce two mappings:
\bea
\label{mappings}
\left.\ba{c}
\dis \Phi_K: \cW_K \mapsto \cV_K,
\q \Psi_K: \cV_K \mapsto \cW_K,\ms\\
\mbox{where}\q \cV_K:=C^{0}_b(\hT_K\times \hR^d; \hR^d), \q \cW_K:= C^{0}_b(\hT_K\times \hR^d\times A; \hR ),
\ea\right.
\eea
such that, for any $(v, w)\in \cV_K\times \cW_K$ and  $(t_k,x,a)\in \hT\times \hR^d\times A$,  recalling \reff{convention} {and \reff{w*rep}},
\bea
\label{phimap}
\left.\ba{lll}
\dis \Phi_K(w)(t_k,x):= \dbE_{t_k,x}\Big[(\nabla \cX^{\D t}_T)^\top g_x(\cX^{\D t}_T) +\sum_{j=k+1}^{K-1}  \lambda\ln \int_A e^{\frac{1}{\lambda}w(t_j,\cX^{\D t}_{t_j},a)}da N^{\D t}_{t_j}\D t\Big];\ms \\
%\dis \Psi_K(v)(t_k,x,a):={1\over \D t} \dbE_{t_k,x}\Big[\int_x^{X^{a}_{t_{k+1}}}  v(t_k,x') dx' -\int_x^{\cX^{\D t}_{t_{k+1}}}  v(t_k,x') dx' \Big]+ r(t_k,x, a).\\
\dis \Psi_K(v)(t_k,x,a):={1\over \D t}  \dbE_{t_k,x}\Big[{\cI(v; t_k,\cX^{\D t}_{t_{k+1}}, X^{a}_{t_{k+1}}\big) }\Big]  + r(t_k,x, a).
\ea\right.
%\label{psimap}
\eea

\ms
\begin{rem}
\label{rem-implementability3}
{\rm {(i) In light of Remarks \ref{rem-implementability}-(i), %we emphasize here again that
the mappings $\Phi_K$ and $\Psi_K$ in (\ref{phimap}) depend only on the coefficients  $\si$ and $g$ which are assumed to be known, as well as the data processes  $X^a$ and $\cX^{\D t}$ along with the reward signals $r(t,x,a)$, which can be either observed or simulated, but not on the functional forms of  $b$ and $r$. This ensures the implementability of any resulting algorithms built on these two mappings.

\ss
(ii) Note that, in the above we discretize $\cX$ but not $X$ because the former can only be simulated while the latter is observed upon query. However, in the numerical examples in \S\ref{sect-numerical} below, we will use simulated data of $X$, and thus will have to discretize $X$ too. Nevertheless, similarly to the estimates in Lemma \ref{lem-Kerr}, such a difference has no impact on the final convergence rate.
}
}\qed
\end{rem}
Denote by $v^*_{\hT_K}:=v^*|_{\hT_K}\in \cV_K$ and $w^*_{\hT_K}:=w^*|_{\hT_K}\in \cW_K$ the restriction of $v^*, w^*$ on $\hT_K$. The following result, whose proof is postponed to Appendix,  rigorously justifies our construction of $(\Phi_K, \Psi_K)$ as a machinery to find  the {\it approximate} fixed point $(v^*, w^*)$ in \reff{v*rep} and \reff{w*rep} and, moreover, reveals that the approximation is of the order of the square root of the time step size ${\D t}$.
\begin{prop}
\label{prop-fixedpoint}
%Under Assumption \ref{assum-standing}, t
There exists a constant $C_{L_*, T}$, which may depend on $L_*, T$ but not on $K$ or $\D t$, such that
\bea
\label{PhiPsiDt}
%\left.\ba{c}
  \|\Phi_K({w^*_{\hT_K}})-v^*_{\hT_K}\|_{0} + \|\Psi_K({v^*_{\hT_K}})-w^*_{\hT_K}\|_{0} \leq  C_{L_*, T}\sqrt{\D t}.
%  \sup_{k} \Big[\|\Phi_K(w^*)(t_k, \cd)-v^*(t_k, \cd)\|_0+\|\Psi_K(v^*)(t_k, \cd)-w^*(t_k, \cd)\|_0\Big] \leq  C\big(\|u^*_x\|_0+\|u^*_{xx}\|_0+\|u^*_{xxx}\|_0\big)\sqrt{\D t}.
%  \ea\right.
\eea

%where $\g(x):=\ln(\frac1x)\sqrt{x}$, $x>0$.
\end{prop}
The next estimate is crucial for building our algorithm.
%Before we prove  Theorem \ref{measure}, let us first establish a simple but important lemma.
\begin{lem}
\label{lem-contraction}
Let %Assumption \ref{assum-standing} hold and
$T\le \bar T$ for some $\bar T<\infty$. There exists a constant $C>0$, depending only on $d$, $\l$, $L_0$, $L_1$, $\bar T$, but not on $L_*$, $T$, $K$,  such that,   for any $(v,w), (v', w')\in \cV_K\times \cW_K$,
%with $T=K\D t$,
%([0, T]\times \hR^d; \hR)$ and $w,w'\in C^{0,0}_b([0, T]\times \hR^d\times A; \hR )$, given any fixed integer $K$ and a even partition as in \reff{partition},
%we have the following estimates
\bea
\label{contr}
\norm{\Phi_K(w)-\Phi_K(w')}_0\leq C\sqrt{T} \norm{w-w'}_0;\q
\norm{\Psi_K(v)-\Psi_K(v')}_0\leq C\norm{v-v'}_0.
\eea
%where $C>0$ is a generic constant independent of $g$.
\end{lem}
\proof  Fix $T\le  \bar T$ and $K\ge 1$. Recall Remarks \ref{rem-implementability}-(v) that we can use the same Brownian motion $B$ in \reff{valueu} and \reff{cXtx}, and Lemma \ref{lem-Kerr} that we can use a generic constant $C=C_{\bar T}$ independent of $T\le \bar T$ for all the estimates therein. Fix an arbitrary triplet $(t_k,x,a)$, and assume without loss of generality that $k=0$. For any $j$,
\beaa
\lambda\ln \int_A e^{\frac{1}{\lambda}w(t_j,\cX^{\D t}_{t_j},a)}da &\le&  \lambda\ln \int_A e^{\frac{1}{\lambda}[w'(t_j,\cX^{\D t}_{t_j},a) +  \|w-w'\|_0]}da \\
&=& \lambda\ln \int_A e^{\frac{1}{\lambda}w'(t_j,\cX^{\D t}_{t_j},a)}da + \|w-w'\|_0.
\eeaa
Similarly one can prove the symmetric inequality, leading to
\bea
\label{gamma-est}
\Big|\lambda\ln \int_A e^{\frac{1}{\lambda}w(t_j,\cX^{\D t}_{t_j},a)}da - \lambda\ln \int_A e^{\frac{1}{\lambda}w'(t_j,\cX^{\D t}_{t_j},a)}da\Big| \le \|w-w'\|_0.
\eea
Then by (\ref{phimap}), \reff{cXdiscretization}, and Lemma \ref{lem-Kerr}, we have
\beaa
&&\bigl|(\Phi_K(w)-\Phi_K(w'))(t_0,x)\bigr| \le C\norm{w-w'}_0 \sum_{j=1}^{K-1}  \hE_{0,x}\big[|N_{t_j}^{\D t}|\big] \D t\\
&& \q\le  C\norm{w-w'}_0\sum_{j=1}^{K-1}  {\D t \over \sqrt{t_j}}\leq C\norm{w-w'}_0 \sqrt{K\D t}= C\sqrt{T}\norm{w-w'}_0.
\eeaa
Next, note that
\beaa
&&\bigl|(\Psi_K(v)-\Psi_K(v'))(t_0,x,a)\bigr|=\frac{1}{\D t}\Bigl|\hE_{0,x}\bigl[{\cI\big(v-v'; t_1, \cX^{\D t}_{t_{1}}, X_{t_{1}}^{a}\big)}\bigr]\Bigr|\\
&&\le \frac{1}{\D t}\norm{v-v'}_0\hE_{0,x}\big[\big|X_{t_{1}}^{a}-\cX^{\D t}_{t_{1}}\big|\big]\\
&&\le \frac{1}{\D t}\norm{v-v'}_0\hE_{0,x}\Big[ \int_{t_0}^{t_{1}}|b(s, X^a_s, a)|ds + \big|\int_{t_0}^{t_{1}}(\si(s, X^a_s) - \si(t_0, x)) dB_s\big|\Big].
\eeaa
Then, it follows from Assumption \ref{assum-standing} and Lemma \ref{lem-Kerr} that
\beaa
&&\bigl|(\Psi_K(v)-\Psi_K(v'))(t_0,x,a)\bigr|\\
&&\le \frac{C}{\D t}\norm{v-v'}_0\Big[\D t + \Big(\dbE_{0,x}\big[\int_{t_0}^{t_{1}}|\si(s, X^a_s) - \si(t_0, x)|^2ds\big] \Big)^{1\over 2}\Big]\\
&&\le \frac{C}{\D t}\norm{v-v'}_0\Big[\D t + \Big(\dbE_{0,x}\big[\int_{t_0}^{t_{1}} (s + |X^a_s - x|^2)ds\big] \Big)^{1\over 2}\Big]\\
&&\le \frac{C}{\D t}\norm{v-v'}_0\Big[\D t + \Big(\int_{t_0}^{t_{1}} sds\big] \Big)^{1\over 2}\Big] \le C \|v-v'\|_0.
\eeaa
By the arbitrariness of $(t_k, x, a)$ we complete the proof.
\qed

%We are now ready to present our new algorithm. We remark that all
The analysis so far holds true for arbitrary $T$. However, the result below will be valid only for small $T$, and we will extend it to general $T$ in the next section.

%Our algorithm solves  $(v, w)$ iteratively.
Fix $K$ and introduce the loss function:
\bea
\label{J}
J_K(v,w):= \norm{\Phi_K(w)-v}_0+ \norm{\Psi_K(v)-w}_0,\q (v, w)\in \cV_K\times \cW_K.
\eea
Proposition \ref{prop-fixedpoint} implies that
\bea
\label{J*}
J_K(v^*_{\hT_K},w^*_{\hT_K}) \le C_{L_*, T} \sqrt{\D t}, \q\mbox{and hence}\q \inf_{(v,w)\in \cV_K\times \cW_K} J_K(v, w) \le C_{L_*, T} \sqrt{\D t}.
\eea
The problem now  boils down to finding a minimizing sequence for $J_K$. %, which can be done, say, by using reinforcement learning algorithms, as we will use in the numerical examples in  \S\ref{sect-numerical}.
We have the following main result.
\begin{thm}
\label{thm-measure}
%Under Assumption \ref{assum-standing}, t
There exist constants $ \delta, C >0$, depending only on $d$,  $\lambda$, $|A|$, and $L_0, L_1$ but not on $L_*$ or $T$, such that for any $K$ and $T\le \delta$,
% for any $(v,w)\in \cV_K\times \cW_K$ and $p\in\big(0,\frac12\big)$, it holds that
\bea
\label{Jest}
\norm{\Phi_K(w)-v^*_{\hT_K}}_0+\norm{\Psi_K(v)-w^*_{\hT_K}}_0\leq  CJ_K(v,w)+ C_{L_*}\sqrt{\D t},~ (v, w)\in \cV_K\times \cW_K.
\eea
%Here in the left side above $v^*, w^*$ are restricted to $\hT_K$ as usual. %where $C>0$ is a generic constant, and $C_g>0$ depends on $g$.
\end{thm}
\proof Fix  $\d >0$ which will be specified later, and assume without loss of generality that $\d\le 1$. Fix $K$ and assume $T\le \d$. Applying Proposition \ref{prop-fixedpoint} and Lemma \ref{lem-contraction} we have
\bea
\label{DPhi*}
\norm{\Phi_K(w)-v^*_{\hT_K}}_0&\le &\norm{\Phi_K(w)-\Phi_K({w^{*}_{\hT_K}})}_0+ \|\Phi_K({w^{*}_{\hT_K}}) - v^*_{\hT_K}\|_0 \nonumber\\
&\le&C\sqrt{\d}\norm{w-{w^*_{\hT_K}}}_0+C_{L_*} \sqrt{\D t}\nonumber\\
&\leq& C\sqrt{\d}\big[\norm{w-\Psi_K(v)}_0+\norm{v-\Phi_K(w)}_0+\norm{\Psi_K(v)-{w^*_{\hT_K}}}_0\big]+C_{L_*} \sqrt{\D t}\nonumber\\
&=& C\sqrt{\d}\big[J_K(v,w)+\norm{\Psi_K(v)-w^*_{\hT_K}}_0\big]+C_{L_*}\sqrt{\D t}.
\eea
Here we used the obvious fact that the constant $C_{L_*, T}$ in \reff{PhiPsiDt} is increasing in $T$ and thus we may set $C_{L_*} = C_{L_*,1}$.
Similarly, we also have
\bea
\label{DPsi*}
\norm{\Psi_K(v)-w^*_{\hT_K}}_0&\leq&\norm{\Psi_K(v)-\Psi_K(v^*_{\hT_K})}_0+ \|\Psi_K(v^{*}_{\hT_K}) - w^*_{\hT_K}\|_0\nonumber\\
&\leq& C\norm{v-v^*_{\hT_K}}_0+C_{L_*}\sqrt{\D t} \nonumber\\
&\leq& C\big[J_K(v,w)+\norm{\Phi_K(w)-v^*_{\hT_K}}_0\big]+C_{L_*}\sqrt{\D t}.
\eea
% Plugging \reff{DPsi*} into \reff{DPhi*}, we obtain  %{we have, for a constant $C_0$ depending only on the model parameters $d$, $\l$, $|A|$, and $L_0, L_1$ in Assumption \ref{assum-standing},}
% \bea
% \label{DPhi*2}
% (1-C^2\sqrt{\delta})\norm{\Phi_K(w)-v^*_{\hT_K}}_0 \le  C(C+1)\sqrt{\d}J_K(v,w)+C_{L_*}(C\sqrt{\d}+1)\sqrt{\D t}.
% \eea
% %\bea
% %\label{DPhi*2}
% %\norm{\Phi_K(w)-v^*_{\hT_K}}_0 \le  {C_0}\sqrt{\d}\Big[J_K(v,w)+\norm{\Phi_K(w)-v^*_{\hT_K}}_0\Big]+C_{L_*}\sqrt{\D t}.
% %\eea
% Set $\d := {1\over 4C^4}\wedge 1$.
Plugging \reff{DPsi*} into \reff{DPhi*}, we have, for a constant $C_0$ depending only on the model parameters $d$, $\l$, $|A|$, and $L_0, L_1$ in Assumption \ref{assum-standing},
\bea
\label{DPhi*2}
\norm{\Phi_K(w)-v^*_{\hT_K}}_0 \le  C_0\sqrt{\d}\Big[J_K(v,w)+\norm{\Phi_K(w)-v^*_{\hT_K}}_0\Big]+C_{L_*}\sqrt{\D t}.
\eea
Set $\d := {1\over 4C_0^2}\wedge 1$ for the above $C_0$. Then,
\beaa
\norm{\Phi_K(w)-v^*_{\hT_K}}_0 \le  {1\over 2}\Big[J_K(v,w)+\norm{\Phi_K(w)-v^*_{\hT_K}}_0\Big]+C_{L_*}\sqrt{\D t}.
\eeaa
Thus
\beaa
{1\over 2}\norm{\Phi_K(w)-v^*_{\hT_K}}_0 \le  {1\over 2} J_K(v,w)+ C_{L_*}\sqrt{\D t}.
\eeaa
% Plugging \reff{DPsi*} into \reff{DPhi*}, { by taking a common large constant $C$ in these two equations, we obatin
% \bea
% \norm{\Phi_K(w)-v^*_{\hT_K}}_0 &\le & C^2\sqrt{\delta}\norm{\Phi_K(w)-v^*_{\hT_K}}_0\nonumber \\
% &&+C(C+1)\sqrt{\d}J_K(v,w)+C_{L_*}(C\sqrt{\d}+1)\sqrt{\D t}.
% \eea
}
% Then we have, for a constant $C_0$ depending only on the model parameters $d$, $\l$, $|A|$, and $L_0, L_1$ in Assumption \ref{assum-standing},
% \bea
% \label{DPhi*2}
% {(1-C_0\sqrt{\d})\norm{\Phi_K(w)-v^*_{\hT_K}}_0 \le  C_0\sqrt{\d}J_K(v,w)+C_{L_*}\sqrt{\D t}}.
% \eea
% Set $\d := {1\over 4C^4}\wedge 1$ for the above $C$. Then, with possibly different values of $C$ and $C_{L_*}$,
% \beaa
% \norm{\Phi_K(w)-v^*_{\hT_K}}_0 \le  CJ_K(v,w)+C_{L_*}\sqrt{\D t}.
% \eeaa
Substituting this into (\ref{DPsi*}) we obtain \reff{Jest} and hence complete the proof.
\qed

The following is the convergence result.
\begin{thm}
\label{thm-conv}
Let %Assumption \ref{assum-standing} hold true and
$T\le \d$ for $\d>0$ from Theorem \ref{thm-measure}. Then for any $\e>0$, there is  $K$ (or equivalently $\D t$) such that $J_K$ has an $\e$-minimizer $(v^\e, w^\e)\in \cV_K\times \cW_K$ that satisfies
\bea
\label{vwe-conv}
\|\bar v^\e - v^*_{\hT_K}\|_0 + \|\bar w^\e - w^*_{\hT_K}\|_0 \le C_{L_*} \e,\q \mbox{where}\q \bar v^\e := \Phi_K(w^\e), ~\bar w^\e := \Psi_K(v^\e).
\eea
%where, again, $v^*, w^*$ are restricted to $\hT_K$.
Moreover, for $k=0,\cds, K$, define
\bea
\label{upie}
\left.\ba{c}
\dis \bar u^\e(t_k, x) :=\dbE_{t_k,x}\Big[ g(\cX^{\D t}_T) +\sum_{j=k}^{K-1}  \lambda\ln \int_A e^{\frac{1}{\lambda}\bar w^\e(t_j,\cX^{\D t}_{t_j},a)}da\D t\Big],\ms\\
\dis \pi^\e(t_k,x,a):= \G(t_k,x,\bar v^\e(t_k,x), a).
\ea\right.
\eea
Then, restricting $u^*, \pi^*$ to $\hT_K$, we have
\bea
\label{upie-conv}
\|\bar u^\e - u^*_{\hT_K}\|_0 + \|\pi^\e - \pi^*_{\hT_K}\|_0 \le C_{L_*} \e.
\eea
\end{thm}
\proof Fix $\e>0$. Let $\D t>0$ be small enough such that $C_{L_*, T} \sqrt{\D t} \le {\e\over 2}$ for the constant $C_{L_*, T}$ in \reff{J*}. Then $\inf_{(v, w)\in \cV_K\times \cW_K}J_K(v,w)\le {\e\over 2}$, and thus $J_K$ has $\e$-minimizer. Now \reff{vwe-conv} follows directly from \reff{Jest}, and \reff{upie-conv} follows from rather standard estimates.
\qed

\begin{rem}
\label{rem-Picard}
{\rm
% {We can use any specific iterative scheme to find a minimizing sequence of $J_K$, such as the Picard iteration or the stochastic gradient descent.}
{In implementation one can use neural networks to approximate the functions $v$ and $w$, and then apply any specific iterative scheme to find a minimizing sequence of $J_K$, such as  Picard's iteration or the stochastic gradient descent (SGD). In \S \ref{5.1}, we will present an SGD based algorithm.} A benefit of the estimate \reff{Jest} is that we can evaluate $J_K(v,w)$ along the iterative steps, which enables us to gauge the error of the current solution to decide when to stop the algorithm.}
\qed
\end{rem}

%\begin{rem}
%\label{rem-elliptic}

To conclude this section, we briefly discuss the infinite horizon case with time homogeneous coefficients $b, \si, r$. In this case,  the exploratory HJB equation \reff{HJBu} becomes elliptic with a certain discount factor $\rho>0$ (cf. \cite{MWZ1}):
\bea
\label{HJBv}
\dis \rho u^*(x) = {1\over 2} [\si\si^\top](x) : u^*_{xx} (x)  + H(x,  u^*_x),\q x\in\hR^d.
\eea
With straightforward  modifications, our analysis remains valid when $\rho$ is large, which essentially corresponds to a small $T$. Indeed, by omitting time discretization, we may modify  the mappings and loss function with the parameter $\D t$ as follows:
\bea
\label{infinity-mapping}
\left.\ba{c}
\dis \Phi(w)(x):= \dbE_{0,x}\Big[\int_0^\infty e^{-\rho t} \ln \int_A e^{w(\cX_t,a)}da N_tdt\Big],\ms\\
\dis \Psi_{\D t}(v)(x,a):={1\over \D t} \dbE\Big[{ \cI(v; \cX_{\D t}, X^{a}_{\D t}\big)} \Big]+ r(x, a),\ms\\
\dis J_{\D t} (v, w) := \|\Phi(w)-v\|_0 + \|\Psi_{\D t}(v) - w\|_0,\q v\in C^0_b(\dbR^d; \dbR^d), w\in C^0_b(\dbR^d\times A; \dbR).
\ea\right.
\eea
{ Here $\cI(v; x_1, x_2\big)$ is defined analogously as in \reff{uv} except now that it is independent of the time variable $t$.}
Then one can similarly show that,
\bea
\label{infinity-est1}
\left.\ba{c}
\dis J_{\D t} (v^*, w^*) \le C_\rho \sqrt{\D t},\ms\\
 \dis \norm{\Phi(w)-\Phi(w')}_0\leq {C\over \sqrt{\rho}} \norm{w-w'}_0;\q
\norm{\Psi_{\D t}(v)-\Psi_{\D t}(v')}_0\leq C\norm{v-v'}_0,
\ea\right.
\eea
and, for $\rho$ sufficiently large,
\bea
\label{infinity-est2}
 \|\Phi(w)-v^*\|_0 + \|\Psi_{\D t} (v) - w^*\|_0 \le C J_{\D t}(v, w) + C\sqrt{\D t}.
\eea
Then, again, the problem boils down to solving the minimization problem $\inf_{v, w} J_{\D t}(v, w)$. We leave the details to interested readers.
Note that, in the infinite horizon setting, we are able to solve a special case of the problem when the diffusion coefficient depends on control, namely $\si=\si(x,a)$; see  \S\ref{sect-volatility}.
%\end{rem}

\section{General Time Horizon }
\label{sect-multi}
\setcounter{equation}{0}
In this section we extend the results in the previous section to the case of a general $T$. For arbitrary $T>0$, let $\d>0$ be determined in Theorem \ref{thm-measure}, which is independent of $T$. Consider a uniform partition: $0=T_0<T_1<...<T_S=T$, with $\D T := {T\over S} \leq \d$ and $T_i:= i\D T$, $i=0,...,S$.
%, where $\d$ satisfy \reff{Jdelta} and is independent of $g$.
Now, for each sub-interval $[T_i,T_{i+1}]$, we introduce a further uniform time discretization $\hT^i_K = \{t^i_k\}_{k=0,\cds, K}$, with $\D t := {\D T\over K} = {T\over SK}$ and $t^i_k := T_i + k \D t$. As in \reff{mappings}, denote
\beaa
%\label{mappings}
%\left.\ba{c}
%\dis \Phi_K: \cW_K \mapsto \cV_K, \q \Psi_K: \cV_K \mapsto \cW_K,\ms\\
\cV^i_K:=C^{0}_b(\hT^i_K\times \hR^d; \hR^d), \q \cW^i_K:= C^{0}_b(\hT^i_K\times \hR^d\times A; \hR ),\q i=0,\cds, m-1.
%\ea\right.
\eeaa

 Our main idea of treating the arbitrary time horizon is to apply the previous analysis/results to a small period recursively  on each sub-interval $[T_i, T_{i+1}]$ in a {\it backward} manner, starting from $[T_{S-1}, T_S]=[T_{S-1}, T]$. To be precise, denote
 \bea
 \label{v*m}
 \bar v^*_S(x) := g_x(x),\q C^*_S := 0,\q\mbox{which is equivalent to}\q \|\bar v^*_S - v^*(T_S,\cd)\|_0 \le C^*_S \sqrt{\D t},
 \eea
 where $v^*$ is the true solution.  For $i=S-1, \cds, 0$, assume we have constructed $\bar v^*_{i+1} \in C^0_b(\dbR^d; \dbR^d)$ and there exists a constant $C^*_{i+1}\ge 0$ such that
 \bea
 \label{v*i+1}
 \|\bar v^*_{i+1} - v^*(T_{i+1},\cd)\|_0 \le C^*_{i+1} \sqrt{\D t}.
 \eea
 Now as in \reff{mappings} and \reff{phimap} we construct mappings $\Phi^i_K: \cW^i_K \mapsto \cV^i_K$ and $\Psi^i_K: \cV^i_K \mapsto \cW^i_K$ by
\bea
\label{phimap-i}
\left.\ba{lll}
\dis \Phi^i_K(w)(t^i_k,x):= \dbE_{t^i_k,x}\Big[(\nabla \cX^{\D t}_{T_{i+1}})^\top \bar v^*_{i+1}(\cX^{\D t}_{T_{i+1}}) +\sum_{j=k+1}^{K-1}  \lambda\ln \int_A e^{\frac{1}{\lambda}w(t^i_j,\cX^{\D t}_{t^i_j},a)}da N^{\D t}_{t^i_j}\D t\Big];\ms \\
%\dis \Psi^i_K(v)(t^i_k,x,a):={1\over \D t} \dbE_{t^i_k,x}\Big[\int_x^{X^{a}_{t^i_{k+1}}}  v(t^i_k,x') dx' -\int_x^{\cX^{\D t}_{t^i_{k+1}}}  v(t^i_k,x') dx' \Big]+ r(t^i_k,x, a),
\dis \Psi^i_K(v)(t^i_k,x,a):={1\over \D t} \dbE_{t^i_k,x}\Big[{ \cI(v; t^i_k, \cX^{\D t}_{t^i_{k+1}}, X^{a}_{t^i_{k+1}}\big)}\Big]+ r(t^i_k,x, a),
\ea\right.
%\label{psimap}
\eea
and define the loss function as in \reff{J} by:
\bea
\label{Ji}
J^i_K(v,w):= \norm{\Phi^i_K(w)-v}_0+ \norm{\Psi^i_K(v)-w}_0,\q (v, w)\in \cV^i_K\times \cW^i_K.
\eea
Then we have the following result.
\begin{thm}
\label{thm-measure-i}
%Let Assumption \ref{assum-standing} hold, and c
Consider the above setting where \reff{v*i+1} holds. Then
\bea
\label{PhiPsiDt-i}
 &&\dis \|\Phi^i_K(w^*_{\hT^i_K})-v^*_{\hT^i_K}\|_0 + \|\Psi^i_K(v^*_{\hT^i_K})-w^*_{\hT^i_K}\|_0 \leq  C_{L_*}\big[ 1+ C^*_{i+1}\big]\sqrt{\D t},\\
% \label{contr-i}
%&&\norm{\Phi^i_K(w)-\Phi^i_K(w')}_0\leq C\sqrt{\d} \norm{w-w'}_0;\q \norm{\Psi^i_K(v)-\Psi^i_K(v')}_0\leq C\norm{v-v'}_0;\\
\label{Jiest}
&&\dis \norm{\Phi^i_K(w)-v^*_{\hT^i_K}}_0+\norm{\Psi^i_K(v)-w^*_{\hT^i_K}}_0\leq  CJ^i_K(v,w)+ C_{L_*}\big[ 1+ C^*_{i+1}\big]\sqrt{\D t},
\eea
for all $i$ and all $(v, w)\in \cV^i_K\times \cW^i_K$.
\end{thm}
 \proof Denote
 \beaa
\tilde  \Phi^i_K(w)(t^i_k,x):= \dbE_{t^i_k,x}\Big[(\nabla \cX^{\D t}_{T_{i+1}})^\top v^*(T_{i+1}, \cX^{\D t}_{T_{i+1}}) +\!\! \sum_{j=k+1}^{K-1}  \!\! \lambda\ln \int_A \!\! e^{\frac{1}{\lambda}w(t^i_j,\cX^{\D t}_{t^i_j},a)}da N^{\D t}_{t^i_j}\D t\Big].
 \eeaa
 Then, since $\D T \le \d\le 1$,  it follows from Proposition \ref{prop-fixedpoint}  that
 \beaa
 \|\tilde \Phi^i_K(w^*_{\hT^i_K})-v^*_{\hT^i_K}\|_0 + \|\Psi^i_K(v^*_{\hT^i_K})-w^*_{\hT^i_K}\|_0 \leq  C_{L_*}\sqrt{\D t}.
 \eeaa
 Moreover, for any $(t^i_k, x)$, by Lemma \ref{lem-Kerr} and \reff{v*i+1} we have
 \beaa
 \big|\tilde  \Phi^i_K(w)(t^i_k,x) -   \Phi^i_K(w)(t^i_k,x)\big| \le \dbE_{t^i_k,x}\big[|\nabla \cX^{\D t}_{T_{i+1}}|\big]  \|\bar v^*_{i+1} - v^*(T_{i+1},\cd)\|_0 \le C_{L_*} C^*_{i+1}\sqrt{\D t}.
\eeaa
 Combining the above two estimates, we obtain \reff{PhiPsiDt-i} immediately.

 Next, Lemma \ref{lem-contraction} remains true for the mappings $\Phi^i_K, \Psi^i_K$. That is,
 \beaa
\norm{\Phi^i_K(w)-\Phi^i_K(w')}_0\leq C\sqrt{\d} \norm{w-w'}_0,\q
\norm{\Psi^i_K(v)-\Psi^i_K(v')}_0\leq C\norm{v-v'}_0,
\eeaa
for all  $(v,w), (v', w')\in \cV^i_K\times \cW^i_K$. Combining this with \reff{PhiPsiDt-i}, we derive \reff{Jiest} following the same arguments as in Theorem \ref{thm-measure}.
\qed

Now set $\e := 2 C_{L_*}\big[ 1+ C^*_{i+1}\big]\sqrt{\D t}$ with the constant $C_{L_*}$ in \reff{PhiPsiDt-i}. Applying  the arguments in proving Theorem \ref{thm-conv}, we obtain an $\e$-minimizer $(v^\e_i, w^\e_i)\in \cV^i_K\times \cW^i_K$ of $J^i_K$. Denoting $\bar v^\e_i := \Phi^i_K(w^\e_i)$, $\bar w^\e_i := \Psi^i_K(v^\e_i)$, we have
\bea
\label{vwe-convi}
\|\bar v^\e_i - v^*_{\hT^i_K}\|_0 + \|\bar w^\e_i - w^*_{\hT^i_K}\|_0 \le C_{L_*} \e =  2|C_{L_*}|^2\big[ 1+ C^*_{i+1}\big]\sqrt{\D t}.
\eea
%Here $v^*, w^*$ are restricted to $\hT^i_K$ for the norm $\|\cd\|_0$, and the constant $C_{L_*}$ varies.
Set
\bea
\label{v*i}
\bar v^*_i(x) := \bar v^\e_i(t^i_0, x) = \bar v^\e(T_i, x),\q C^*_i := 2|C_{L_*}|^2\big[ 1+ C^*_{i+1}\big].
\eea
We conclude that  $\bar v^\e_i$ and $C^*_i$ satisfy \reff{v*i+1} at $i$, thus completing the induction step.

We are now ready to state the main result of this section.
\begin{thm}
\label{thm-multi}
Let %Assumption \ref{assum-standing} hold, and let
$\d$, $T_i$, $\hT^i_K, \e$ be given and  $v^\e_i, w^\e_i$, $\bar v^\e_i$, $\bar w^\e_i$, $v^*_i$, $C^*_i$, $J^i_K$ be constructed as above, where in particular $(v^\e_i, w^\e_i)$ is an $\e$-minimizer of $J^i_K$. Moreover, define $\bar u^*_S(x) := g(x)$, and for $i=S-1, \cds, 0$, and for all $k$,
\bea
\label{upie-i}
\left.\ba{c}
\dis \bar u^\e_i(t^i_k, x) :=\dbE_{t^i_k,x}\Big[ \bar u^*_{i+1}(\cX^{\D t}_{T_{i+1}}) +\sum_{j=k}^{K-1}  \lambda\ln \int_A e^{\frac{1}{\lambda}\bar w^\e_i(t^i_j,\cX^{\D t}_{t^i_j},a)}da\D t\Big],\q \bar u^*_i(x) := \bar u^\e_i(t^i_0, x);\ms\\
\dis \pi^\e_i(t^i_k,x,a):= \G(t^i_k,x,\bar v^\e_i(t^i_k,x), a).
\ea\right.
\eea
Then, there exists $C_{L_*}>0$ such that, for $i=0,\cds, m-1$,
\bea
\label{multi-est}
\left.\ba{c}
\dis \|\bar v^\e_i - v^*_{\hT^i_K}\|_0 + \|\bar w^\e_i - w^*_{\hT^i_K}\|_0 \le C^*_i \sqrt{\D t};\ms\\
\dis  \|\bar u^\e_i - u^*_{\hT^i_K}\|_0\le 2C^*_i \sqrt{\D t},\q  \|\bar \pi^\e_i - \pi^*_{\hT^i_K}\|_0 \le C_{L_*} C^*_i \sqrt{\D t}.
\ea\right.
\eea
\end{thm}
\proof The estimates for $\bar v^\e_i, \bar w^\e_i$ are already given in \reff{vwe-convi}. To estimate $\bar u^\e_i$, first note that $\|\bar u^*_n - u^*(T_S,\cd)\|_0=0 = 2C^*_S \sqrt{\D t}$. Assume we have proved $\|\bar u^*_{i+1} - u^*(T_{i+1},\cd)\|_0 \le 2C^*_{i+1} \sqrt{\D t}$. Recall \reff{u*rep} and note that, for $k=0,\cds, K-1$,
\beaa
u^*(t^i_k,x)=\hE_{t^i_k,x}\Big\{u^*(T_{i+1}, \cX_{T_{i+1}}) + \int_{t^i_k}^{T_{i+1}} \lambda\ln\big[\int_A \exp \big(\frac{1}{\lambda}w^*(s,\cX_s,a)\big)da\big]ds\Big\}.
\eeaa
Then, by \reff{upie-i} and \reff{gamma-est} we have, for the function $H^*$ defined in \reff{H*} in Appendix,
\beaa
&\dis \big|\bar u^\e_i(t^i_k, x) - u^*(t^i_k,x)\big| \le \|\bar u^*_{i+1} - u^*(T_{i+1},\cd)\|_0 + \sum_{j=k}^{K-1} \|\bar w^\e_i - w^*_{\hT^i_K}\|_0 \D t  + \sum_{j=k}^{K-1}\int_{t^i_j}^{t^i_{j+1}}I^j_sds,\\
&\dis \mbox{where}\q I^j_s:= \hE_{t^i_k,x} \big[\big|H^*(s,\cX_s)-H^*(t^i_j,\cX^{\D t}_{t^i_j})\big|\big].
\eeaa
It follows from the regularity of $H^*$; see \reff{H*est} in Appendix, along with Lemma \ref{lem-Kerr}, that
\beaa
|I^j_s|\le C_{L_*}\hE_{t^i_k,x} \big[\sqrt{s-t^i_j} + |\cX_s - \cX_{t^i_j}| + | \cX_{t^i_j} -  \cX^{\D t}_{t^i_j}\big] \le C_{L_*}\sqrt{\D t}.
\eeaa
Then, noting $K\D t = \D T \le \d\le 1$, we get
\beaa
 \big|\bar u^\e_i(t^i_k, x) - u^*(t^i_k,x)\big| \le 2C^*_{i+1}\sqrt{\D t}+ [C^*_{i}\sqrt{\D t} + C_{L_*}\sqrt{\D t}\big] K \D t \le 2 C^*_i\sqrt{\D t}.
\eeaa
In particular,
\beaa
 \|\bar u^*_{i} - u^*(T_{i},\cd)\|_0 = \sup_x \big|\bar u^\e_i(t^i_0, x) - u^*(t^i_0,x)\big| \le  2 C^*_i\sqrt{\D t}.
\eeaa
This completes the induction step and thus we obtain the estimate for $\bar u^\e_i$.

Finally, recall \reff{pi*} and note that
\beaa
\left.\ba{lll}
\dis \pa_z \g(t,x,z,a):= \frac{1}{\lambda}b(t,x, a)\g(t,x,z,a),\ms\\
\dis  \pa_z \G(t, x, z, a) := \frac{\frac{1}{\lambda}b(t,x, a)\g(t,x,z,a)}{\int_A \g(t,x,z,a') d a'} - \frac{\g(t,x,z,a)\int_A \frac{1}{\lambda}b(t,x, a')\g(t,x,z,a') d a'}{(\int_A \g(t,x,z,a') d a')^2}.
\ea\right.
\eeaa
Then,
\beaa
\left.\ba{lll}
\dis  \big|\pa_z \G(t, x, z, a)\big| \le \frac{C\g(t,x,z,a)}{\int_A \g(t,x,z,a') d a'}  + \frac{C\g(t,x,z,a)\int_A \g(t,x,z,a') d a'}{(\int_A \g(t,x,z,a') d a')^2} = C \G(t, x, z, a).
\ea\right.
\eeaa
Moreover, noting  $|v^*|\le C_{L_*}$ as well as  the estimate of $\bar v^\e_i$ we have
\beaa
|\bar v^\e_i(t^i_k,x)| \le |v^*(t^i_k,x)| + C^*_i\sqrt{\D t} \le C_{L_*} + 1.
\eeaa
Here we assume without loss of generality that $\D t$ is small enough so that $C^*_i\sqrt{\D t} \le 1$. Then by Assumption \ref{assum-standing} and \reff{pi*} we can easily see that $\G(t, x, z, a) \le C_{L_*}$ for all $z$ between $\bar v^\e_i(t^i_k,x)$ and  $v^*(t^i_k,x)$. Thus
\beaa
\big|\pi^\e_i(t^i_k,x,a) - \pi^*(t^i_k,x,a)\big| &=& \Big|\G(t^i_k,x,\bar v^\e_i(t^i_k,x), a)-\G(t^i_k,x, v^*(t^i_k,x), a)\Big| \\
&\le& C_{L_*} \big|\bar v^\e_i(t^i_k,x) - v^*(t^i_k,x)\big| \le C_{L_*} C^*_i \sqrt{\D t}.
\eeaa
This completes the proof.
\qed

\begin{rem}
\label{rem-largeT}
{ \rm Recall that the constant $\d>0$ does not depend on $T$ or $L_*$. For a general time horizon $T>0$, the partition number  $S$ is determined by ${T\over \d}$. So when $T$ is very large and/or when $\d$ is very small ( when $L_0, L_1$ are large), $S$ will be large as well, so will $C^*_i$. Indeed, from \reff{v*i} it follows that $C^*_0 \sim (2|C_{L_*}|^2)^S \sim  (2|C_{L_*}|^2)^{T\over \d}$. In other words,  theoretically our scheme may become inefficient when $T$ is really large. Nevertheless, in the next section we will show numerically that for reasonably large $T$ or $S$ this multi-step procedure is indeed efficient while the one-step method in the previous section simply fails.}
\qed
\end{rem}

\makeatletter
\newenvironment{breakablealgorithm}
  { \begin{center}
     \refstepcounter{algorithm}%
     \hrule height .8pt depth0pt \kern 2pt
     \renewcommand{\caption}[2][\relax]{%
       {\raggedright\textbf{Algorithm \thealgorithm} ##2\par}%
       \ifx\relax##1\relax
         \addcontentsline{loa}{algorithm}{\protect\numberline{\thealgorithm}{##2}}%
       \else
         \addcontentsline{loa}{algorithm}{\protect\numberline{\thealgorithm}{##1}}%
       \fi
       \kern2pt\hrule\kern2pt
     }%
  }
  { \kern2pt\hrule\end{center} }
\makeatother

\section{Numerical Experiments}\label{sect-numerical}
\setcounter{equation}{0}
In this section, we first present our numerical algorithms, including a single-step algorithm for a small time horizon $T$ and a multi-time-period algorithm for a general $T$, along with numerical examples in both cases, including high-dimensional ones. Finally, we give comparisons with some well-known existing algorithms.

\subsection{Algorithm for small time horizon }\label{5.1}

{
Before presenting the full numerical scheme in Algorithm \ref{alg:vanilla-single}, we clarify several key points. %design choices to bridge the theory with our numerical implementation.

{First, the ultimate goal of the algorithm is to learn the optimal value function $u^*$ (which solves the exploratory HJB) and the optimal exploratory policy $\pi^*$. In view of \reff{u*rep} and \reff{pi*}, it suffices to learn the functions $v^*$ and $w^*$. In the algorithm, we use two neural networks, $v^{\psi}$ and $w^{\phi}$ to approximate $v^*$ and $w^*$ respectively.

Next, the algorithm takes a batch of $M$ state/reference trajectories, each starting from the  initial time $t_0=0$ and an initial state $x_0$ randomly drawn from a distribution $\mathrm{Law}(X_0)$, chosen based on the specific objective of the task at hand. If the objective is to solve the exploratory HJB and/or obtain the optimal feedback policy, then a typical choice of $\mathrm{Law}(X_0)$ is a uniform distribution supported on a bounded region (that is likely to contain the most states of interest). If one aims to just learn the optimal value $u^*(0,x_0)$ with respect to a specific, given initial state $x_0$, then the distribution can be just the Dirac measure centered at that initial state $x_0$. % to evaluate the problem at a fixed initial state $(t_0, x_0)$; or a uniform distribution over a specified interval, to approximate the global spatial profile $u(t_0, \cdot)$.

In the procedure we need to apply a constant control $a$ drawn uniformly from the action space $A$. In general when $A$ is an arbitrary measurable set with a non-zero and finite volume (which is assumed in this paper), there is a so-called hit-and-run algorithm that can sample uniformly from $A$; see \cite{S,BRS}. Alternatively, especially when $A$ has a more regular shape (e.g. an orthotope), we can use a sufficiently fine uniform grid $\{a_j\}_{j=1}^{N_a}$ to approximate $A$, where $N_a$ denotes the number of discrete points.
Then we can sample actions uniformly just from these grid points. This method, which we will employ in our subsequent numerical experiments, has an added benefit: it also provides the quadrature nodes required to numerically evaluate the integral in the $\Phi$ mapping.   Aligning the sampling grid with the integration grid ensures consistency across iterations and simplifies the evaluation of the policy.}

In view of Remark \ref{rem-implementability}-(v), we need to carefully manage the independent Brownian increments used to evaluate the mappings $\Phi$ and $\Psi$. The  mapping $\Phi$ in \reff{phimap-i} is evaluated along an {\it entire}  reference path $\cX^{(m)}$ driven by the increments $B_k^{(m)}$; see Algorithm \ref{alg:vanilla-single}, line 10-13.  By contrast,  the mapping $\Psi$, at each node $\cX_k^{(m)}$, requires {\it one-step} transitions of both the controlled process $X_{k+1}^{a,(m,n)}$ and the reference path $\cX_{k+1}^{(m,n)}$.
In practice, the former is observed from the environment upon a query (see Remark \ref{rem-implementability}-(i),(ii)) while the latter is simulated using $\Tilde{B}_k^{(m,n)}$ that is independent from $B_k^{(m)}$; see Algorithm \ref{alg:vanilla-single}, lines 17-20. In our numerical experiments below, we simulate the environment using Brownian increments independent of those driving the controlled state or reference processes. We note that, in practical applications, the overall cost of the algorithm also includes that of acquiring the controlled-transition data $X_{k+1}^{a,(m,n)}$.

Finally, we formalize the function approximation and the empirical loss. Recall we need to learn two networks $v^{\psi}$ and $w^{\phi}$. The theoretical objective relies on the supremum norm, which is well-suited for convergence analysis but impractical for gradient-based optimization due to its non-smoothness. Therefore, we approximate the supremum norm using a differentiable softmax surrogate. Given a parameter $\tau>0$, for any vector $l$ of length $K$, we define:
$$
\mathrm{softmax}_\tau(l) \;=\; \tau\,\log\!\sum_{i=1}^{K}\exp(l_i/\tau).
$$
It is a standard result that $\mathrm{softmax}_\tau(l)$ converges to the supremum norm of $l$ as $\tau\to 0$.\footnote{Another natural approximation of the max norm is the $\hL^p$ norm with a relatively large $p$. However, when $p$ becomes large, computations become noticeably heavier in both time and memory, as observed in our preliminary experiments (not reported in the paper). Specifically, $p\geq 8$ introduces significant slowdowns even in the one-dimensional case.}

To construct the empirical loss function, we first compute the residuals on each trajectory $(m)$ at each time step $t_k$:
\beaa
l_k^{v,(m)}&=&\big|\,v^\psi(t_k,\cX^{(m)}_k)-\Phi^{(m)}_k\,\big|,\\
l^{w,(m)}_k &=& \big|\,w^\phi(t_k,\cX^{(m)}_k,a^{(m)})-\Psi^{(m)}_k\,\big|.
\eeaa
Applying the softmax approximation, we aggregate these residuals over the time grid (from $k=0$ to $K-2$) to obtain the trajectory-level losses:
$$
L_v^{(m)} = \mathrm{softmax}_\tau\left(\{l^{v,(m)}_k\}_{k=0}^{K-2}\right),\quad L_w^{(m)} = \mathrm{softmax}_\tau\left(\{l^{w,(m)}_k\}_{k=0}^{K-2}\right).
$$
Averaging across the batch of size $M$, the total empirical loss $J$, which proxies the theoretical supremum norm objective, is given by:
\beaa
J &=& \frac{1}{M}\sum_{m=1}^{M} \left( L_v^{(m)} + L_w^{(m)} \right).
\eeaa
We update the parameters $\psi$ and $\phi$ to minimize $J$ via stochastic gradient descent with learning rates $\eta_v$ and $\eta_w$ respectively.
}

We now present our first algorithm, Algorithm \ref{alg:vanilla-single}. For ease of presentation, we write the pseudocode for $d=1$. Modifications are straightforward for $d>1$. Take the expression of $\nabla\cX$ for example, it will be a matrix-valued Jacobian initialized at $I_d$. Its
 scalar exponential update will be replaced by a discretization of
the corresponding matrix variational SDE.
Moreover, the integrand involved will become
$(\sigma^{-1}\nabla\cX)^\top dB$, while the terminal derivative term will be
$(\nabla\cX)^\top\nabla g$.

\begin{breakablealgorithm}
\caption{ (Single-time-period algorithm)}
\label{alg:vanilla-single}
\begin{algorithmic}[1]\small

\State \textbf{Input:}
Fixed initial time $t_0=0$, terminal time $T$,
number of grid points $K$, batch size $M$, nested path size $N$,
and epochs $E$; action space $A$, %action grid $\{a_j\}_{j=1}^{N_a}$,
temperature parameter $\lambda>0$, softmax parameter $\tau$,
and learning rates $\eta_v,\eta_w$;
 distribution of the initial state
$\mathrm{Law}(X_0)$.

\State \textbf{Output:} Trained networks $(v^\psi,w^\phi)$ and the resulting
estimated optimal value $\hat u$ and policy $\hat\pi$.

\State \textbf{Initialize:}
\Statex \quad Set
\[
\D t\gets\frac{T-t_0}{K-1},
\qquad
t_k\gets t_0+k\D t,
\qquad k=0,\dots,K-1,
\]
and randomly initialize the network weights $(\psi,\phi)$
for $v^\psi(t,x)$ and $w^\phi(t,x,a)$.

\For{$\mathrm{epoch}=1,\dots,E$}
    \For{$m=1,\dots,M$}
        \com{Batch generation}

        \State Draw
        $\cX_0^{(m)}\sim\mathrm{Law}(X_0)$.

        \State Draw primary Brownian increments
        \[
        \{\D B_k^{(m)}\}_{k=0}^{K-2}
        \sim\cN(0,\D t).
        \]

        \State Set
        \[
        \nabla\cX_0^{(m)}\gets1,
        \qquad
        N_0^{(m)}\gets0.
        \]

        \For{$k=0$ \textbf{to} $K-2$}
            \com{Forward simulation of the primary reference path}

            \State
            \q
            $\cX_{k+1}^{(m)}
            \gets
            \cX_k^{(m)}
            +
            \sigma(t_k,\cX_k^{(m)})\D B_k^{(m)}$.

            \State
            \q
            $\nabla\cX_{k+1}^{(m)}
            \gets
            \nabla\cX_k^{(m)}
            \exp\!\left(
                \sigma_x(t_k,\cX_k^{(m)})\D B_k^{(m)}
                -
                \frac12
                \sigma_x^2(t_k,\cX_k^{(m)})\D t
            \right)$.

            \State
            \q
            $N_{k+1}^{(m)}
            \gets
            N_k^{(m)}
            +
            \nabla\cX_k^{(m)}
            \sigma^{-1}(t_k,\cX_k^{(m)})
            \D B_k^{(m)}$.

        \EndFor

        \For{$k=0$ \textbf{to} $K-2$}
            \com{Compute $\Phi$- and $\Psi$-residuals}

            \State Compute $\Phi_k^{(m)}$ using the primary-path data
            $\left(
            k,
            \{\cX_i^{(m)}\},
            \{N_i^{(m)}\},
            \{\nabla\cX_i^{(m)}\};
            \phi
            \right).
            $
            \State
            $
            l_k^{v,(m)}
            \gets
            \left|
            v^\psi(t_k,\cX_k^{(m)})
            -
            \Phi_k^{(m)}
            \right|.
            $

            \For{$n=1,\dots,N$}
                \com{Nested paths from $(t_k,\cX_k^{(m)})$}

                \State Draw
                $\D\Tilde B_k^{(m,n)}
                \sim
                \cN(0,\D t)$.
                \State Simulate the nested reference transition
                $$
                \cX_{k+1}^{(m,n)}
                \gets
                \cX_k^{(m)}
                +
                \sigma(t_k,\cX_k^{(m)})
                \D\Tilde B_k^{(m,n)}.
                $$
                \State Sample
                $
                a^{(n)}
                \sim
                \mathrm{Unif}(A).
                %\bigl(\{a_j\}_{j=1}^{N_a}\bigr).
                $
                \State Apply the fixed action $a^{(n)}$ at
                $(t_k,\cX_k^{(m)})$ and observe
                $
                X_{k+1}^{a^{(n)},(m,n)}.
                $
            \EndFor

            \State Compute $\Psi_k^{(m)}$ using the nested data
            $\left(
            k,\cX_k^{(m)},a^{(n)},
            \{X_{k+1}^{a^{(n)},(m,n)}\},
            \{\cX_{k+1}^{(m,n)}\};
            \psi
            \right)$.

            \State
            $
            l_k^{w,(m)}
            \gets
            \left|
            w^\phi(t_k,\cX_k^{(m)},a^{(n)})
            -
            \Psi_k^{(m)}
            \right|.
            $
        \EndFor

        \State
        $
        L_v^{(m)}
        \gets
        \mathrm{softmax}_\tau
        \left(
        \{l_k^{v,(m)}\}_{k=0}^{K-2}
        \right).
        $

        \State
        $L_w^{(m)}
        \gets
        \mathrm{softmax}_\tau
        \left(
        \{l_k^{w,(m)}\}_{k=0}^{K-2}
        \right)$.

    \EndFor

    \State Compute the total empirical loss
    \[
    J
    \gets
    \frac1M
    \sum_{m=1}^M
    \left(
    L_v^{(m)}+L_w^{(m)}
    \right).
    \]

    \State Update
    \[
    \psi
    \gets
    \psi-\eta_v\nabla_\psi J,
    \qquad
    \phi
    \gets
    \phi-\eta_w\nabla_\phi J.
    \]
\EndFor

\State Obtain $\hat\pi(t,x,a)$ from $w^\phi(t,x,a)$
through \reff{pi*}.

\State Recover $\hat u$ from $w^\phi(t,x,a)$ using
independent paths through \reff{u*rep}.

\State \textbf{Return:}
$(v^\psi,w^\phi,\hat u,\hat\pi)$.

\end{algorithmic}
\end{breakablealgorithm}

We now present numerical results for applying Algorithm \ref{alg:vanilla-single}.
%\begin{rem}
\rm{Unless otherwise stated, all the test results reported are based on 50 independent runs on a local machine with an NVIDIA GeForce RTX 3080 Ti GPU (12 GB GDDR6X, 384-bit memory interface). For detailed neural network specifications, hyperparameters, and the complete codebase, please refer to our GitHub repository https://github.com/GaozhanWang/MWZZ-RL-2026.} %\qed
%\end{rem}

We start with two one-dimensional stochastic control examples.
\begin{eg}\label{EX-base} {\rm Consider a control problem with the following coefficients and control set:
\beaa
b(t,x,a) = x+a,\quad \sigma(t,x)=1, \quad a \in [0,1],
\eeaa
along with the reward functions
\beaa
r(t,x,a)=(2t+2ax+1)e^{-(t^2 + x^2 + 1)},\quad g(x) = e^{-(T^2 + x^2 - 1)}.
\eeaa
%Because the dynamics coefficients are linear, the standing Assumptions \ref{assum-standing} and \ref{data}-(i) are satisfied.
%Unlike the usual case, where people define the running reward and then try to solve for the value function, for testing purposes, we set up the value function first as
%$$
%u^*(t,x) = e^{-(t^2 + x^2 + 1)}.
%$$
%Then we find the corresponding running reward under the given dynamics that satisfies the exploratory HJB equation \reff{HJBu} is given as
%$$r(t,x,a)=(2t+2ax)u^*.$$
Given those coefficients, setting  $\lambda =1$, we can solve the exploratory HJB equation \reff{HJBu} analytically to get the {\it ground-truth} optimal value function
as well as the corresponding optimal policy
$$
u^*(t,x) = e^{-(t^2 + x^2 + 1)};\qq \pi^*(t,x,a)=\frac{1}{|A|}=1,\q  a\in[0,1].
$$
Now we pretend that we do not know the analytical forms of $b$ and $r$, and apply Algorithm 1 to solve the problem numerically. First consider the problem starting from
$(t_0,x_0)=(0,0.1)$, in which case we choose $ \mathrm{Law(X_0)}= \mathrm{Dirac}(x_0)$. In implementation, we set  $T=0.4, \D t = 0.02$. %All of our following reported results are based on 50 independent runs on each example. We
Denote %the true value and the predicted value as $u^*(t_0,x_0),\hat{u}(t_0,x_0)$ respectively, and denote
the relative error to be
$$RE=\Big|\frac{u^*(t_0,x_0)-\hat{u}(t_0,x_0)}{u^*(t_0,x_0)}\Big|,$$
where $u^*(t_0,x_0),\hat{u}(t_0,x_0)$ are the true (initial) value and the learned one respectively. This error is motivated by a scenario where one is concerned with only the optimal value corresponding to the particular initial pair $(t_0,x_0)$.
% The following table reports the testing results based on 50 independent runs on a local machine with an NVIDIA GeForce RTX 3080 Ti GPU (12\,GB GDDR6X, 384-bit memory interface).
% \begin{table}[h]
%   \centering

%   \label{tab:num-summary-grid}
%   \begin{tabular}{|c|c|c|c|c|}
%     \hline
%     \text{$u^*(t_0,x_0)$} &
%     \text{Mean $\hat{u}(t_0,x_0)$} &
%     \text{Mean RE} &
%     \text{STD RE} &
%     \text{Mean runtime (s)} \\
%     \hline
%     0.3642 & 0.3546 & 2.6397\% & 0.4459\% & 38.12 \\
%     \hline
%     % 1.2345 & 1.2410 & 0.009 & 0.0001 & 11.80 \\
%     % \hline
%     % 0.5432 & 0.5498 & 0.015 & 0.0003 & 18.07 \\
%     % \hline
%     % --- Add more rows as needed ---
%   \end{tabular}
%   \caption{Numerical results for EX1.1.}
% \end{table}
% %0.4 11 64 32 96 40
\begin{table}[H]
  \centering
  \begin{tabular}{|c|c|c|c|c|c|}
    \hline
    \text{ Time horizon}&
    \text{$u^*(t_0,x_0)$} &
    \text{Mean $\hat{u}(t_0,x_0)$} &
    \text{Mean RE} &
    \text{STD RE} &
    \text{Mean runtime (s)} \\
    \hline
    $T$=0.4&  0.3642 & 0.3546 & 2.6397\% & 0.4459\% & 38.12 \\
    \hline
    % T=0.5&0.3642 & 0.3395 & 6.7679\% & 0.4714\% & 49.97 \\
    % \hline
    % 1.2345 & 1.2410 & 0.009 & 0.0001 & 11.80 \\
    % \hline
    % 0.5432 & 0.5498 & 0.015 & 0.0003 & 18.07 \\
    % \hline
    % --- Add more rows as needed ---
  \end{tabular}
  \caption{Numerical results with Algorithm
  \ref{alg:vanilla-single} on Example \ref{EX-base}.}
\end{table}
%0.4 11 64 32 96 40
%0.5 11 64 32 96 40

Applying Algorithm 1 with the choice of $ \mathrm{Law(X_0)}=\mathrm{Dirac}(x_0)$, we not only can learn the particular value $\hat{u}(t_0,x_0)$ but also the {\it entire} optimal value function $\hat u$ along with the optimal policy $\hat\pi$. At below we %Figure \ref{fig:pi_surface}
provide visualizations of the thus learned optimal value function and optimal policy versus the oracle ones in the following two figures. To be more precise, Figure \ref{fig:u_surface} shows the learned value
function surface and its relative error with respect to the oracle
solution.\footnote{Due to our GPU limit, we display only the
portion of the surface over $t\in[0,0.3]$ for illustration. All the figures in this paper are provided in vector format; hence
although scaled down to fit the page layout, they can be
enlarged without loss of resolution to inspect finer details.}
The three plots in Figure \ref{fig:pi_surface} compare the learned
policy $\hat{\pi}(t,\cdot,\cdot)$ with the true optimal policy
$\pi^*(t,\cdot,\cdot)$ at the time slices $t=0,0.2,0.4$, respectively.

\begin{figure}[H]
    \centering
    \includegraphics[width=0.80\textwidth]{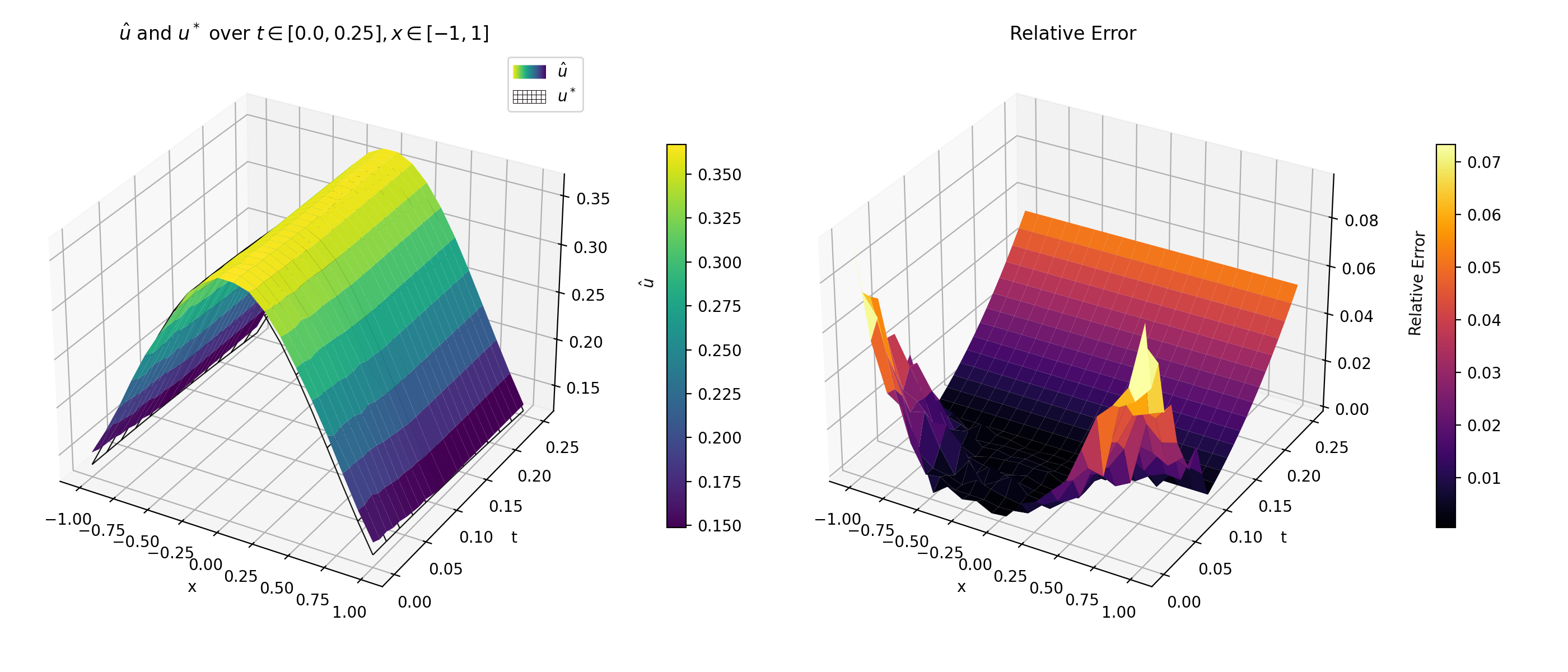}
    \caption{$\hat{u}$ vs $u^*$}
    \label{fig:u_surface}
\end{figure}
\begin{figure}[h]
    \centering
    \includegraphics[scale=0.21]{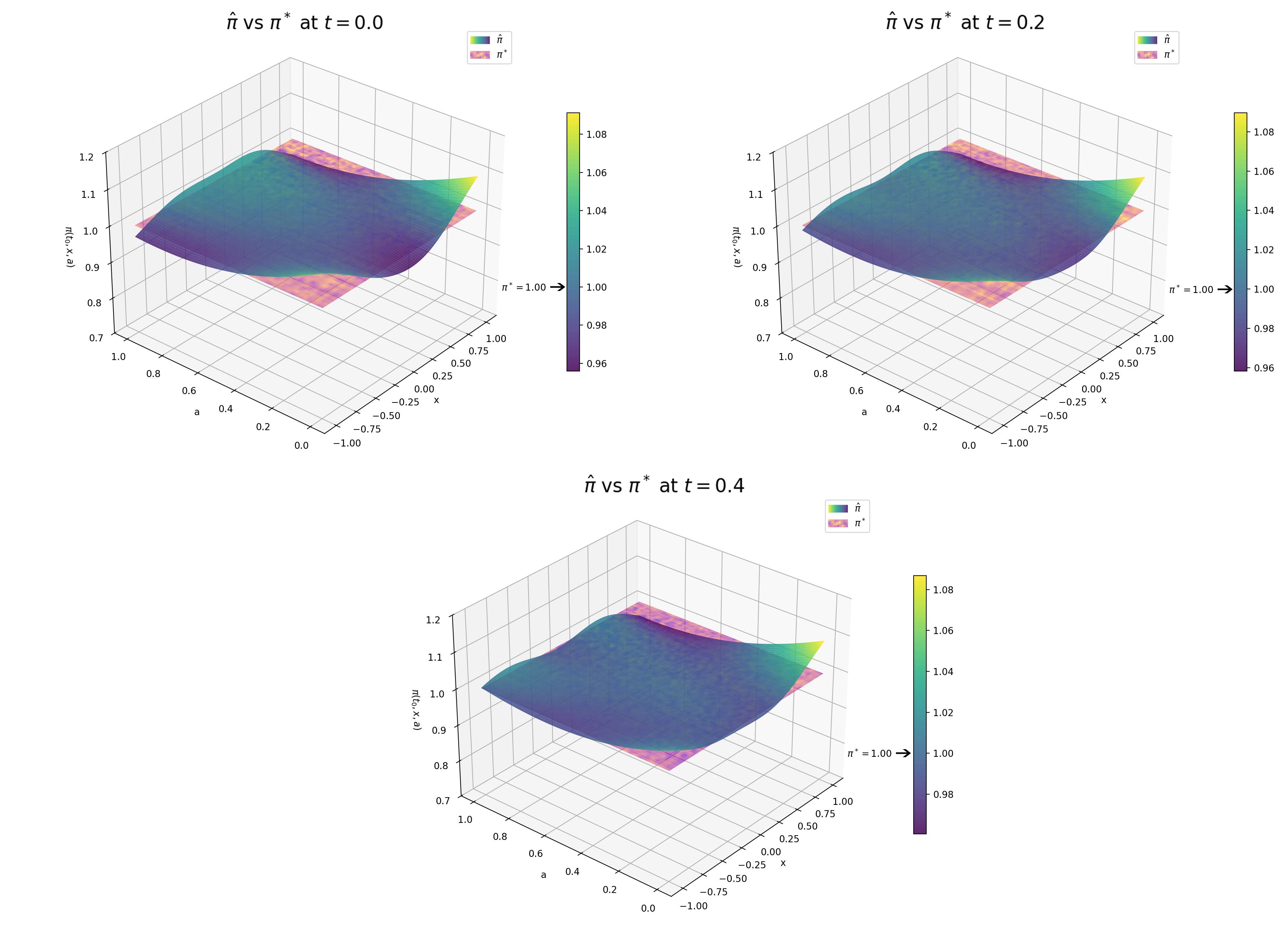}
    \caption{$\hat{\pi}$  vs $\pi^*$}
    \label{fig:pi_surface}
\end{figure}
We observe in Figure \ref{fig:u_surface} that the relative errors in the optimal value function remain small in a neighborhood of the given initial point, $(t_0,x_0)=(0.0,0.1)$. As state variable $x$ moves away from this initial condition $x_0$, the errors gradually increase. This behavior is natural
because in this example $\hat u$ is learned based on state/reference trajectories all initialized at $(t_0,x_0)=(0.0,0.1)$. Consequently, the training data are concentrated along trajectories emanating from this point, leading to higher accuracy in its vicinity and weaker generalization in more distant regions. For applications where the control problem is strictly tied to a specific initial time and state, this localized sampling approach provides a highly computationally efficient solution and avoids the unnecessary overhead of exploring the broader state space.

If the goal of the learning is to obtain the {\it global} optimal value function (or equivalently to solve the exporatory HJB) and/or the optimal feedback policy, then we need to randomize the initial state $x_0$ starting from $t_0=0$ according to some non-Dirac distribution $\mathrm{Law}(X_0)$ for generating training trajectories/data as discussed earlier.
For illustration and simplicity, we  take $\mathrm{Law}(X_0)$ to be a uniform distribution over an interval $[x^{\min}_0, x^{\max}_0]=[-1,1]$ for the current example. The corresponding test results (run with a new set of  independent random seeds) are presented in Figure \ref{fig:u_curve}.
\begin{figure}[H]
    \centering
    \includegraphics[scale=0.5]{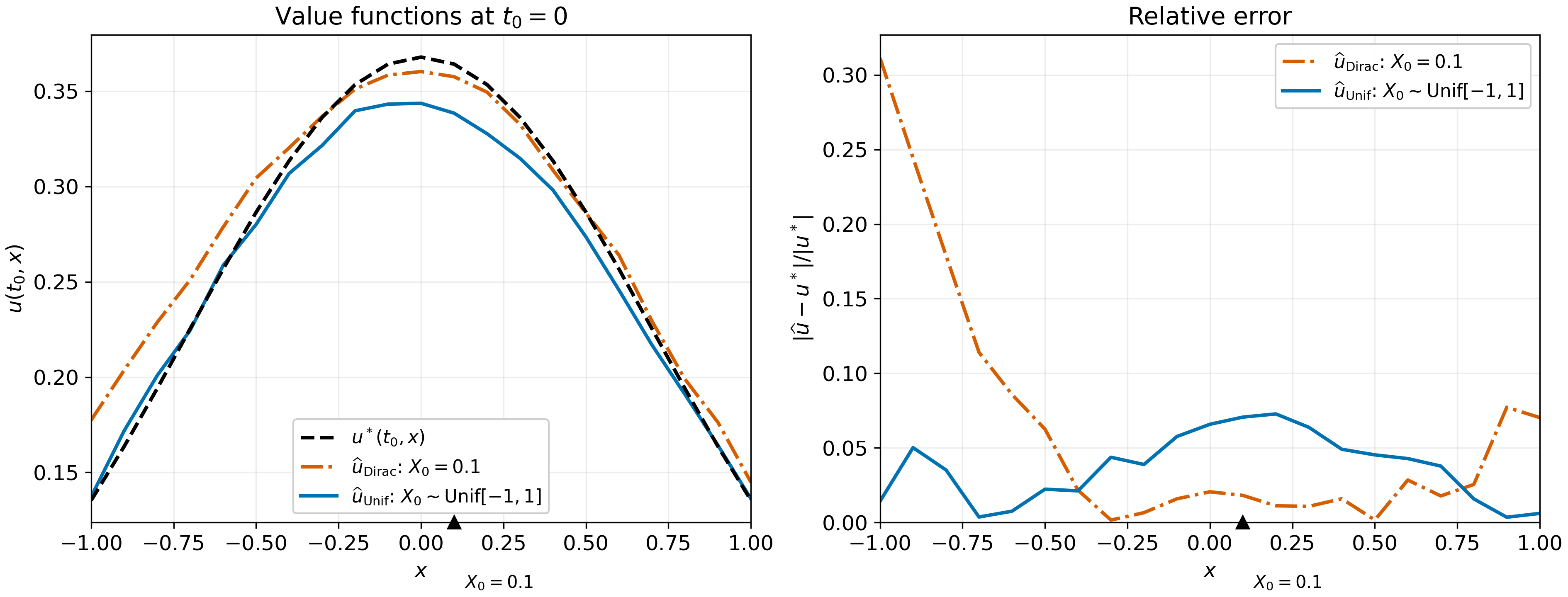}
    \caption{$\hat{u}(0,\cdot)$  vs $u^*(0,\cdot)$}
    \label{fig:u_curve}
\end{figure}

While randomizing initial state naturally demands a higher computational cost during training, it successfully mitigates the local concentration of trajectories. Indeed, Figure \ref{fig:u_curve} shows that the uniform accuracy of the approximation is now significantly improved across the target state space.

} \end{eg}
\begin{eg}\label{EX-tri} {\rm The next example involves trigonometric type functions with a non-constant volatility. Specifically,
\beaa
b(t,x,a) = -2\cos(t+x)-\frac{1}{4}\sin(2t+2x)-1+a,\quad \sigma(t,x)=2+\sin(t+x), \quad a \in [0,1],
\eeaa
with
$$
r(t,x,a)=a\sin(t+x)+2\cos(t+x), \q g(x) = \cos(T+x).
$$
Setting $\lambda =1$, the theoretical optimal value function is $u^*(t,x) = \cos(t+x)$. We set $(t_0,x_0)=(0,0),T=0.1,\D t = 0.02$, and obtain the results below:
% 0.1，5，64，64，96，40
\begin{table}[H]
  \centering
  \begin{tabular}{|c|c|c|c|c|c|}
    \hline
    \text{Time horizon}&
    \text{$u^*(t_0,x_0)$} &
    \text{Mean $\hat{u}(t_0,x_0)$} &
    \text{Mean RE} &
    \text{STD RE} &
    \text{Mean runtime (s)} \\
    \hline
    $T$=0.1 & 1.0000 &  0.9737 & 2.6222\% & 0.0423\% & 29.15 \\
    \hline
    % 1.2345 & 1.2410 & 0.009 & 0.0001 & 11.80 \\
    % \hline
    % 0.5432 & 0.5498 & 0.015 & 0.0003 & 18.07 \\
    % \hline
    % --- Add more rows as needed ---
  \end{tabular}
  \caption{Numerical results with Algorithm \ref{alg:vanilla-single} on Example \ref{EX-tri}.}
\end{table}

% 0.1，1，6，64，64，96

Note that in this example as both the model parameters and the truth value functions are trigonometric, we deliberately avoid using \texttt{tanh} as the neural network activation function—even though it is common—in order to test the generality of our algorithm’s behavior.
}
\end{eg}

We now use the previous two examples to compare Algorithm \ref{alg:vanilla-single} with the classical model-based method which assumes all the model parameters to be known and applies PIA to directly approximate the optimal value function. Importantly, as discussed earlier PIA takes auto-diff to directly calculate the gradient of the approximated value function across iterations.

\begin{eg}\label{EX-PIA} {\rm (PIA Comparison) We find that the model-based PIA is numerically unstable: with GPU acceleration, it either terminates quickly in a few seconds but stops at a poor estimation, or it fails to converge within a reasonable runtime.  In the tables below, we only report results with ``converged" values and, in the latter case, we cap the policy iteration count at 1000 to prevent endless runs. As a result, we do not report runtimes here.

Tables \ref{tab:num-summary-grid} and \ref{tab:num-summary-grid2} indicate the poor performance of the PIA scheme, which generally does not converge to the true solution.

\begin{table}[htbp]
  \centering
  \begin{tabular}{|c|c|c|c|c|}
    \hline
    \text{Time Horizon}&
    \text{$u^*(t_0,x_0)$} &
    \text{Mean $\hat{u}(t_0,x_0)$} &
    \text{Mean RE} &
    \text{STD RE} \\
    \hline
    T=0.4&0.3642 & 2.191 & 495.6\% & 6.5301\% \\
    \hline
    % 1.2345 & 1.2410 & 0.009 & 0.0001 & 11.80 \\
    % \hline
    % 0.5432 & 0.5498 & 0.015 & 0.0003 & 18.07 \\
    % \hline
    % --- Add more rows as needed ---
  \end{tabular}
  \caption{Numerical results with PIA for Example \ref{EX-base}.}
  \label{tab:num-summary-grid}
\end{table}

\begin{table}[H]
  \centering
  \begin{tabular}{|c|c|c|c|c|}
    \hline
    \text{Time Horizon}&
    \text{$u^*(t_0,x_0)$} &
    \text{Mean $\hat{u}(t_0,x_0)$} &
    \text{Mean RE} &
    \text{STD RE} \\
    \hline
    T=0.1&1.000 & 1.466 & 46.58\% & 1.4452\% \\
    \hline
    % 1.2345 & 1.2410 & 0.009 & 0.0001 & 11.80 \\
    % \hline
    % 0.5432 & 0.5498 & 0.015 & 0.0003 & 18.07 \\
    % \hline
    % --- Add more rows as needed ---
  \end{tabular}
  \caption{Numerical results with PIA for Example \ref{EX-tri}.}
  \label{tab:num-summary-grid2}
\end{table}

}
\end{eg}

\begin{eg}\label{EX-hdim}{\rm We now extend our numerical experiment to multi-dimensional cases. Take a scaled version of Example \ref{EX-base}, with the dynamics and the true value function given as
\beaa
{b_i(t,\Vec{x},a) = \frac{\Vec{x}_i}{d}+a_0, \q \sigma(t,x)=\hI_d, \quad a=a_0\hI_d,\q a_0\in [0,1], \quad i=1,...,d,}
\eeaa
\beaa
u^*(t,\Vec{x})=e^{-(t^2+\norm{\Vec{x}}^2_2/d+1)},\q \Vec{x}\in \hR^d,
\eeaa
along with the reward functions
{ \beaa
r(t,\Vec{x},a)=(2t+\frac2da^\top\Vec{x}+1)e^{-(t^2 + \norm{\Vec{x}}^2_2/d + 1)},\quad g(\Vec{x}) = e^{-(T^2 + \norm{\Vec{x}}^2_2/d + 1)}.
\eeaa}
We test Algorithm \ref{alg:vanilla-single} on dimensions $d =1,5,10,20,50,100$ with $\Vec{x}_0 = (0.1, 0.1, \dots, 0.1)^\top \in \mathbb{R}^d$, $T= 0.1$, and report the results in {Table \ref{tab:num-summary-grid3}}.\footnote{Algorithm \ref{alg:vanilla-single} is written for $d=1$, but see the remarks immediately before the algorithm pseudocode for an explanation on multi-dimensional cases. Meanwhile, the dramatic runtime jump from dimension 50 to 100 is chiefly a hardware effect—our local GPU saturates its parallel resources (memory-bandwidth limits)—rather than a sudden increase in statistical or algorithmic complexity. Due to the heavy runtime, the results for the 100-dimensional case are based on only 2 independent runs.}
\begin{table}[H]
  \centering
  \begin{tabular}{|c|c|c|c|c|c|}
    \hline
    \text{Dimension} &
    \text{$u^*(t_0,x_0)$} &
    \text{Mean $\hat{u}(t_0,x_0)$} &
    \text{Mean RE} &
    \text{STD RE} &
    \text{Mean runtime (s)} \\
    \hline
    $d$=1&0.3642 & 0.3659 & 0.4747\% & 0.1503\% & 10.20 \\
    \hline

    $d$=5&0.3642 & 0.3507 & 3.6524\% & 0.1282\% & 11.54 \\
    \hline
 %0.1 5 64 32 96 20
    $d$=10&0.3642 & 0.3400& 6.6441\% &0.2399\% & 19.40  \\
    \hline
%0.1 6 96 96 96 40
     $d$=20&0.3642 & 0.3324 & 8.7398\% & 0.2965\% &  29.55 \\
    \hline
%0.1 6 96 128 96 40
     $d$=50&0.3642 & 0.3280 & 9.9332\% & 0.1499\% &  140.21 \\
    \hline
%0.1 5 128 128 96 20
    $d$=100&0.3642 & 0.3266& 10.3333\% & 0.0091\% & 5284.68 \\
    \hline
%0.1 5 128 128 96 20
    % --- Add more rows as needed ---
  \end{tabular}
  \caption{Numerical results for Algorithm \ref{alg:vanilla-single} on Example \ref{EX-hdim}.}
\label{tab:num-summary-grid3}
\end{table}
}
\end{eg}

% We now compare the above result with a state-of-the-art benchmark -- a {\it model-based} method for high-dimensional PDEs in \cite{E2018}.
% %\begin{eg}\label{EX7}
% Under the same setting of Example \ref{EX-hdim} but assuming oracle knowledge of the model parameters, we apply the algorithm in \cite{E2018} to obtain results in the following table.
% \begin{table}[H]
%   \centering
%   \begin{tabular}{|c|c|c|c|c|c|}
%     \hline
%     \text{Dimension} &
%     \text{$u^*(t_0,x_0)$} &
%     \text{Mean $\hat{u}(t_0,x_0)$} &
%     \text{Mean RE} &
%     \text{STD RE} &
%     \text{Mean runtime (s)} \\
%     \hline
%     $d$=5&0.3642& 0.3705 & 2.1920\% & 2.6294\% &  11.04 \\
%     \hline
% %0.1 6 96 96 96 40
%      $d$=10&0.3642 & 0.3631 & 0.4451\% & 0.2767\% & 9.46  \\
%     \hline
%     $d$=20&0.3642 & 0.3635 & 0.2114\% & 0.0767\% &  11.47 \\
%     \hline
% %0.1 6 96 128 96 40
%      $d$=50&0.3642 & 0.3630 & 0.3327\% & 0.2606\% & 31.84  \\
%     \hline
%     $d$=100&0.3642 & 0.3620& 0.6282\% & 0.6418\% &  44.86 \\
%     \hline
%     \end{tabular}
% %0.1 5 128 128 96 20
%   \caption{Numerical results with algorithm in \cite{E2018} for Example \ref{EX-hdim}.}
%   \label{tab:num-ex7}
% \end{table}

{
%As expected, the model-based method of \cite{E2018}, which has oracle access to all model parameters, achieves greater accuracy and less runtime than our model-free method. Nevertheless,
Our algorithm maintains a relative error of about 10\% or less even as the dimension increases to 100. Moreover, the errors appear to level off in higher dimensions, while the very small standard deviations indicate stable performance across independent runs. The estimates exhibit a systematic downward bias as the dimension increases, suggesting a potential direction for further improvement, but the overall results demonstrate that the proposed method remains effective and scalable even with key model primitives missing.\footnote{{
%We should also mention that, unlike CPU execution, where computational time scales roughly linearly with sample size and dimension, GPU parallelization (which we use here)  allows for sub-linear time scaling, maintaining high speed until hardware limits are reached.
We also note that, unlike sequential CPU execution—where runtime grows proportionally with the total computational workload—GPU parallelization absorbs the increased tensor operations across dimensions and batch sizes, maintaining a relatively flat runtime overhead until hardware memory limits are reached. As the dimension increases to $d = 100$, we encounter exactly this hardware
bottleneck under our current local computing power.
%Specifically, for any fixed dimension,our approach requires a significantly larger GPU memory footprint than the model-basedbaseline.
}}}

%% Comparing Tables \ref{tab:num-summary-grid3} and \ref{tab:num-ex7}, we see that our method, even in the absence of the knowledge of some key parameters, achieves comparable accuracy as the model-based benchmark in roughly double or three times the runtime. %, holding all other factors the same.
%% {(I am running a d=100 case for fair comparison. If the runtime does not leap as in our algorithm, we may want to change the presentation to avoid misleading.)
%{
%A comparison of the results in Tables \ref{tab:num-summary-grid3} and \ref{tab:num-ex7} is in order. First of all, due to the missing key model coefficients our algorithm needs to carry out a significantly larger number of sampling and computations that inevitably increase the runtime. On the other hand, using the oracle knowledge of the coefficients, the model-based benchmark achieves a lower mean relative error, especially in high dimensions. Moreover, unlike CPU execution, where computational time scales roughly linearly with sample size and dimension, GPU parallelization allows for sub-linear time scaling, maintaining high speed until hardware limits are reached. As the dimension increases to $d=100$, we encounter exactly this hardware bottleneck under our current local computing power. Specifically, for any fixed dimension, our approach requires a significantly larger GPU memory footprint than the model-based baseline.
%}

%\end{eg}

\subsection{Algorithm for general time horizon}
We now extend Algorithm~\ref{alg:vanilla-single} to the general setting when $T$ is arbitrary. For all the previously reported examples, we will show that the performance of Algorithm~\ref{alg:vanilla-single} deteriorates rapidly as $T$ increases (see Example \ref{EX-multi} below). In particular, it happens when $T>0.4$ in Example \ref{EX-base} and $T>0.1$ in Example \ref{EX-tri}.

{We first present the following (multi-period) algorithm based on our analysis in \S \ref{sect-multi} for a general time horizon $T$}.

\begin{breakablealgorithm}
\caption{(Multi-time-period algorithm)}
\label{alg:multi-train}
\begin{algorithmic}[1]\small

\State \textbf{Input:} Fixed initial time $t_0=0$ and segment partition
\[
t_0=T_0<T_1<\cdots<T_S=T;
\]
\Statex for each segment $j$, number of grid points $K_j$,
epochs $E_j$, and learning rates $\eta_{v,j},\eta_{w,j}$; batch size $M$, nested path size $N$,
action space $A$,
%action grid $\{a_l\}_{l=1}^{N_a}$,
temperature parameter $\lambda>0$, and softmax parameter $\tau$; distribution of the initial state
$\mathrm{Law}(X_0)$.

\State \textbf{Output:} Trained segment networks
$\{(v^{\psi_j},w^{\phi_j})\}_{j=0}^{S-1}$
and the resulting estimated optimal value $\hat u$
and policy $\hat\pi$.

\State \textbf{Initialize:} For each $j=0,\dots,S-1$, set
\[
\D t_j
\gets
\frac{T_{j+1}-T_j}{K_j-1},
\qquad
t_{j,k}
\gets
T_j+k\D t_j,
\qquad k=0,\dots,K_j-1,
\]
and randomly initialize the segment-network weights
$(\psi_j,\phi_j)$.

\For{$j=S-1,\dots,0$}
    \com{Backward induction over segments}

    \State Define the terminal derivative for segment $j$ by
    \[
    g_x^{(j)}(x)
    =
    \begin{cases}
        g_x(x),
        & j=S-1,\\[3pt]
        v^{\psi_{j+1}}(T_{j+1},x),
        & j<S-1.
    \end{cases}
    \]

    \For{$\mathrm{epoch}=1,\dots,E_j$}

        \State \textbf{Sample initial states for the current segment:}

        \For{$m=1,\dots,M$}
            \If{$j=0$}
                \State Draw
                $\cX_{j,0}^{(m)}
                \sim
                \mathrm{Law}(X_0)$.

            \Else
                \State Draw
                $\cX_{j,0}^{(m)}
                \sim
                \mathrm{Law}
                \left(
                \cX_{T_j}^{t_0,X_0}
                \right)$.
                \com{Empirical law obtained from forward simulation}
            \EndIf
        \EndFor
        \State Execute the primary-path simulation, nested-path simulation, and residual computation as in
        Algorithm~\ref{alg:vanilla-single}, restricted to the interval $[T_j,T_{j+1}]$ with time step $\D t_j$. Replace the terminal derivative target in Algorithm~\ref{alg:vanilla-single} by
        $
        g_x^{(j)}
        \left(
        \cX_{j,K_j-1}^{(m)}
        \right).
        $
        \State Aggregate the batch residuals to obtain
        the segment loss $J_j$.

        \State Update
        \[
        \psi_j
        \gets
        \psi_j-\eta_{v,j}\nabla_{\psi_j}J_j,
        \qquad
        \phi_j
        \gets
        \phi_j-\eta_{w,j}\nabla_{\phi_j}J_j.
        \]
    \EndFor

    \State Store the trained segment networks
    $(v^{\psi_j},w^{\phi_j})$.
\EndFor

\State Construct $\hat\pi(t,x,a)$ by stitching
$\{w^{\phi_j}\}_{j=0}^{S-1}$ through \reff{pi*}.

\State Recover $\hat u$ using independent paths and
$\{w^{\phi_j}\}_{j=0}^{S-1}$ through \reff{u*rep}.

\State \textbf{Return:}
$\left(
\{v^{\psi_j},w^{\phi_j}\}_{j=0}^{S-1},
\hat u,\hat\pi
\right)$.

\end{algorithmic}
\end{breakablealgorithm}

{
There is a key design choice that distinguishes the multi-period scheme from the single-period one,  Algorithm \ref{alg:vanilla-single}, where the initial state distribution, $\mathrm{Law}(X_0)$, is typically {\it fixed} as, e.g.,  a Dirac or uniform distribution. In the multi-period setting, however,  one needs to utilize the learned value function of period $j$ as the terminal condition for period $j-1$. To incorporate this backward coupling, at the starting point of period $j$, we draw initial states from the empirical law of the forward reference paths $\cX^{t_0,x_0}_{t_j}$ (see line 12 of Algorithm \ref{alg:multi-train}). This mechanism is both natural and canonical, so as to match the distributional weighting inherent to the Feynman–Kac representation (\ref{u*rep}), thereby strictly aligning the training objective of the algorithm with the underlying theory.
}

\begin{eg}\label{EX-multi}
{\rm We now compare Algorithm~\ref{alg:vanilla-single} with Algorithm~\ref{alg:multi-train} using Examples \ref{EX-base} and \ref{EX-tri} with different $T$'s while keeping   all  other parameters (including  $\D t$ and the neural network hyperparameters) unchanged.  Algorithm~\ref{alg:vanilla-single} is trained on a single period, while Algorithm~\ref{alg:multi-train} is trained on multiple periods dividing the given single period.
% For instance, for Examples \ref{EX-base} with $T=0.2$, Algorithm~\ref{alg:vanilla-single} applies to the single segment $[0,0.2]$, whereas Algorithm~\ref{alg:multi-train} splits the horizon into two periods, $[0,0.1]$ and $[0.1,0.2]$, in implementing the algorithm.
The following two tables { show that Algorithm~\ref{alg:multi-train} clearly outperforms Algorithm~\ref{alg:vanilla-single}  in terms of learning accuracy. Also, the former reduces  runtime substantially  for larger $T$}.
\begin{table}[H]
\centering
\begin{tabular}{c|c}
\textbf{Case I} &\textbf{Case II}   \\
Example \ref{EX-base}&Example \ref{EX-base}\\
$T=0.6$&$T=0.8$\\
\text{Single time segment:} $[0,0.6]$&\text{Single time segment:} $[0,0.8]$\\
\text{Multi time segments:} $[0,0.2],[0.2,0.4],[0.4,0.6]$&\text{Multi time segments:} $[0,0.2],...,[0.6,0.8]$
\end{tabular}
\end{table}
\begin{table}[H]
\centering
\begin{tabular}{|c|c|c|c|c|c|}
    \hline
    \text{Time segments} &
    \text{$u^*(t_0,x_0)$} &
    \text{Mean $\hat{u}(t_0,x_0)$} &
    \text{Mean RE} &
    \text{STD RE} &
    \text{Mean runtime (s)} \\
    \hline
    Case I Single&0.3642 &0.3170  & 12.9657\% & 0.8409\% & 72.35 \\
    \hline

    \textbf{Case I Multi}&\textbf{0.3642} & \textbf{0.3542} & \textbf{2.7890\%} & \textbf{1.1365\%} &\textbf{ 121.98} \\
    \hline
    Case II Single&0.3642 &0.2717& 25.3914\% & 0.7311\% &   609.40\\
    \hline
    \textbf{Case II Multi}&\textbf{0.3642 }& \textbf{0.3676} & \textbf{1.3194\% }& \textbf{2.4641\% }& \textbf{147.29} \\

    \hline
    % --- Add more rows as needed ---
\end{tabular}
% 0.6，1，18，64，32，96，100 [-1,1]
% 0.6，3，6，64，64，96，100[-1,1]
% 0.8，4，6，64，64，96，100[-1,1]
% 0.8，1，24，64，64，96，100[-1,1]
\caption{Comparison of two algorithms on Example \ref{EX-base}.}
\label{singlevsmulti-base}
\end{table}

\begin{table}[H]
\centering
\begin{tabular}{c|c}
\textbf{Case III} &\textbf{Case IV}   \\
Example \ref{EX-tri}&Example \ref{EX-tri}\\
$T=0.2$&$T=0.3$\\
\text{Single time segment:} $[0,0.2]$&\text{Single time segment:} $[0,0.3]$\\
\text{Multiple time segments:} $[0,0.1],[0.1,0.2]$&\text{Multiple time segments:}$[0,0.1],[0.1,0.2],[0.2,0.3]$
\end{tabular}
\end{table}
\begin{table}[H]
\centering
\begin{tabular}{|c|c|c|c|c|c|}
    \hline
    \text{Time segment} &
    \text{$u^*(t_0,x_0)$} &
    \text{Mean $\hat{u}(t_0,x_0)$} &
    \text{Mean RE} &
    \text{STD RE} &
    \text{Mean runtime (s)} \\
    \hline
    Case III Single&1.0000 & 0.8771 & 12.2906\% & 0.0642\% & 54.60 \\
    \hline
    \textbf{Case III Multi}&\textbf{1.0000} & \textbf{0.9673 }& \textbf{3.2677\%} & \textbf{0.9841\%} & \textbf{87.19 }\\
    \hline
    Case IV Single&1.0000 &  0.7755& 22.4452\% & 1.161\% &  456.89\\
    \hline
    \textbf{Case IV Multi}&\textbf{1.0000} & \textbf{1.0267} & \textbf{ 2.8463\%} & \textbf{1.2735\%} & \textbf{135.78} \\

    \hline
\end{tabular}
\caption{Comparison of two algorithms on Example \ref{EX-tri}.}
\label{singlevsmulti-tri}
\end{table}
% 0.2，1，11，64，64，96
% 0.2，2，6，64，64，96，[-0.9,0.9]
% 0.3，3，6，64，64，96，[-0.9,0.9]
}
\end{eg}

\section{A Special Case with Controlled Diffusion Coefficients}
\label{sect-volatility}
The general problem considered so far excludes the case when control enters into the diffusion coefficients of the dynamics, for which the approach we develop fails.
In this section, we discuss a special case with controllable diffusion coefficients and one-dimensional state space in infinite time horizon and offer a numerical example. {The treatment is based on  the main idea of  \cite{MWZ1}.}

Recall the infinite horizon problem formulated at the end of \S \ref{sect-small} and consider only the one-dimensional case, i.e., $d=1$.   The entropy-regularized control problem with control-dependent diffusion is
\bea
\label{scalarV}
\left.\ba{lll}
\dis dX^\pi_t = \tilde b(X^\pi_t, \pi(t,X^\pi_t))dt + \sqrt{\widetilde {\si^2}(X^\pi_t, \pi(t, X^\pi_t))} dW_t;\ss\\
\dis J(x, \pi):= \hE\Big[\int_0^\infty e^{-\rho t} \big[\tilde r(X^\pi_t, \pi(t, X^\pi_t)) + \l \cH(\pi(t, X^\pi_t))\big]dt\Big|X^\pi_0=x\Big]; \ss\\
\dis u^*(x) := \sup_{\pi\in \cA} J(x, \pi).
\ea\right.
\eea
The corresponding exploratory HJB is
\bea
\label{scalarHJBv}
\left.\ba{c}
\dis \rho u^* = H(x,  u^*_x,  u^*_{xx}), \q x\in\hR,\q\mbox{where}\ss\\
\dis
H(x, z, q) := \sup_{\pi\in \cP_0(A)} \Big[{1\over 2} \widetilde {\si^2}(x, \pi) q  + \tilde b(x, \pi) z+ \tilde r(x, \pi) + \l \cH(\pi)\Big].
\ea\right.
\eea
The optimal feedback policy  has the Gibbs form:
\bea
\label{scalarHGamma}
\left.\ba{c}
\dis \pi^*(x,a) = \G(x, u^*_x, u^*_{xx}, a),\q \mbox{where}~\G(x, z, q, a) := \frac{\g(x,z, q, a)}{\int_A \g(x,z, q, a') d a'},\ss\\
\dis
\g(x,z, q, a):= \exp \Big(\frac{1}{\lambda}[{1\over 2}\si^2(x,a)q+b(x,a)z+r(x, a)]\Big).
\ea\right.
\eea
One can consequently rewrite \reff{scalarHJBv} as
\beaa
\dis \rho u^*(x)=\lambda\ln\int_A \exp{\Bigl\{\frac{1}{\lambda}\bigl[\frac{1}{2}\sigma^2(x,a)\partial_{xx}u^*(x)+b(x,a)\partial_xu^*(x)+r(x,a)\bigr]\Bigr\}}da.
\eeaa

{%To facilitate our discussion, in the rest of this section we shall make use of the following strengthened standing assumptions. In particular,
We make the following strengthened standing assumptions, including in particular the so-called {\it smallness} assumption on the dependence of the diffusion coefficient $\si$ on control which is used also in e.g. \cite{TWZ}.}
\begin{assum}
\label{assum-vol}There exist constants $L_*,\varepsilon>0$ such that the followings hold true:\\
\ss
(i) $d=1;b,r$ satisfy Assumption \ref{assum-standing};\\
\ss
(ii) %The volatility term has the following special structure
$\sigma^{2}(x,a)=\sigma_0^2(x)+\sigma^2_1(x,a),$ where $\sigma_0$ satisfies Assumption \ref{assum-standing} and $|\sigma_1^2|\leq \varepsilon$;\\
\ss
(iii) The true solution $u^*$ of the HJB equation  \reff{scalarHJBv} satisfies
\beaa
\norm{u^*}_0+\norm{u^*_x}_0+\norm{u^*_{xx}}_0\leq L_*.
\eeaa
\end{assum}

%\begin{rem}
 Throughout this section, we omit the time discretization analysis as it is completely parallel to the one in previous sections, focusing on the key ideas for designing the algorithm. As a result, in contrast to Assumption~\ref{assum-standing} we do not require the boundedness of $\norm{u_{xxx}^*}_0$ here, because the convergence analysis presented here does not involve time discretization errors. Moreover, the boundedness of $\|u^*\|_0$ is not an additional assumption, as it follows directly from the boundedness of the reward function $r$.
%\end{rem}

Under Assumption \ref{assum-vol}, the corresponding HJB equation becomes
\bea \label{scalarHJB2}
\dis \rho u^*(x)=\frac{1}{2}\sigma_0^2(x)u_{xx}^*(x)+\lambda\ln\int_A \exp{\Bigl\{\frac{1}{\lambda}\bigl[\frac{1}{2}\sigma_1^2(x,a)u_{xx}^*(x)+b(x,a)u_x^*(x)+r(x,a)\bigr]\Bigr\}}da.
\eea
Denote
\beaa
v^*(x):=u_x^*(x),\q \theta^*(x):=u_{xx}^*(x), \q w^*(x,a):=b(x,a)u_x^*(x)+r(x,a).
\eeaa
Following the same idea as in \S \ref{sect-idea} by applying the Feynman--Kac formula and Bismut--Elworthy--Li representation, we first have
\bea \label{u*repvol}
&&\dis \!\!\!\!\!\! \!\!\!\!\!\! \!\!\!\!\!\! u^*(x)=\dbE\biggl[\int_0^\infty e^{-\rho t}\lambda\ln\int_A \exp{\Bigl\{\frac{1}{\lambda}\bigl[\frac{1}{2}\sigma_1^2(\cX^x_t,a)\theta^*(\cX^x_t)+w^*(\cX^x_t,a)\bigr]\Bigr\}}dadt\biggr],\\
\label{v*repvol}
&&\dis \!\!\!\!\!\! \!\!\!\!\!\! \!\!\!\!\!\! v^*(x)=\dbE\biggl[\int_0^\infty e^{-\rho t}N_t^x\lambda\ln\int_A \exp{\Bigl\{\frac{1}{\lambda}\bigl[\frac{1}{2}\sigma_1^2(\cX^x_t,a)\theta^*(\cX^x_t)+w^*(\cX^x_t,a)
\bigr]\Bigr\}}dadt\biggr],
\eea
where $\cX^x$ is the dynamics of the reference state in infinite horizon:
\bea
\label{volref0}
\cX_t^x&:=&x+\int_0^t\sigma_0(\cX_s^x)dB_s, \q t>0.
\eea
The kernel term in the representation of $v^*$ in \reff{v*repvol} becomes
\bea\label{Ntvol}
\dis N_t^x=\frac{1}{t}\int_0^t\sigma_0^{-1}(\cX_s^x)\td \cX_s^{x}dB_s,\q \td \cX_t^x = 1 + \int_0^t \pa_x \si_0(\cX^x_s) \td \cX^x_s dB_s.
\eea
Mimicking  the estimate \reff{w*rep} in \S \ref{sect-idea}, we have\footnote{ In the current scalar setting, one can easily see that $\cI(v; x_1, x_2) = \int_{x_1}^{x_2} v(x) dx$.}
\bea\label{w*repvol}
%\dis w^*(x,a)&\approx&{1\over \D t} \dbE\Big[\Big(\int_x^{X^{x,a}_{\D t}}  - \int_x^{\cX_{\D t}^{x,a}}\Big)  v^*(x') dx' + r(x, a) \D t\Big],
\dis w^*(x,a)&\approx&{1\over \D t} \dbE\Big[{\cI\big(v^*; \cX_{\D t}^{x,a}, X^{x,a}_{\D t}\big)} + r(x, a) \D t\Big],
\eea
where $\cX^{x,a}$ is a new, control-dependent reference process defined as
\bea\label{volref1}
\cX^{x,a}_t&:=& x+\int_0^t \sigma(\cX^{x,a}_s,a)dB_s, \q t>0,
\eea
while the controlled state dynamics is
\bea\label{volref2}
X^{x,a}_t&:=& x+\int_0^t b(X^{x,a}_s,a)ds+\int_0^t \sigma(X^{x,a}_s,a)dB_s, \q t>0.
\eea

Finally, thanks to the special structure of $\sigma^2$ and the simplicity in the one-dimensional case, we have the following representation of the second-order derivative term,
\bea\label{theta*repvol}
\theta^*(x)&=&\frac{2\rho}{\sigma_0^2(x)} u(x)-\frac{2}{\sigma_0^2(x)}\biggl[ \lambda\ln\int_A \exp{\Bigl\{\frac{1}{\lambda}\bigl[\frac{1}{2}\sigma_1^2(x,a)\theta^*(x)+w^*(x,a)\bigr]\Bigr\}}da\biggr].
\eea
{Denote
\[
\cU:=C_b^0(\hR;\hR),\qq \cV:=C_b^0(\hR;\hR),\qq \cW:=C_b^0(\hR\times A;\hR), \qq
\Theta:=C_b^0(\hR;\hR).
\]
In light of the relations given by \reff{u*repvol}, \reff{v*repvol}, \reff{w*repvol}, and \reff{theta*repvol}, we introduce four mappings
\bea\label{mappingsvol}
\left.\ba{l}
\Lambda: \cW \times \Theta \mapsto \cU,\q \Phi:\cW\times\Theta \mapsto\cV,\q\Psi:\cV\mapsto\cW,\q
\Pi:\cU\times\cW\times\Theta\mapsto\Theta,
 %\cW:=C_b^0(\hR\times A;\hR), \qq\cV:=C_b^0(\hR;\hR),\\
%\cU:=C_b^0(\hR;\hR),\qq\Theta:=C_b^0(\hR;\hR)
\ea\right.
\eea
as follows:
for any $(u,v,w,\th)\in \cU\times\cV\times\cW\times\Theta$ and $(x,a)\in[0,\infty)\times A$, %recalling \reff{volref1}, \reff{volref2},
\bea\label{volmaps}
\left.\ba{lll}
\dis \Lambda(w,\th)(x):=\dbE\biggl[\int_0^\infty e^{-\rho t}\lambda\ln\int_A \exp{\Bigl\{\frac{1}{\lambda}\bigl[\frac{1}{2}\sigma_1^2(\cX^x_t,a)\theta(\cX^x_t)+w(\cX^x_t,a)\bigr]\Bigr\}}dadt\biggr],\ms\\
\dis \Phi(w,\th)(x):=\dbE\biggl[\int_0^\infty e^{-\rho t}N_t^x\lambda\ln\int_A \exp{\Bigl\{\frac{1}{\lambda}\bigl[\frac{1}{2}\sigma_1^2(\cX^x_t,a)\theta(\cX^x_t)+w(\cX^x_t,a)\bigr]\Bigr\}}dadt\biggr],\ms\\
%\dis \Psi(v)(x,a):={1\over \D t} \dbE\Big[\Big(\int_x^{X^{x,a}_{\D t}}  - \int_x^{\cX_{\D t}^{x,a}}\Big)  v(x') dx' + r(x, a) \D t\Big],\ms\\
\dis \Psi(v)(x,a):={1\over \D t} \dbE\Big[{\cI\big(v; \cX_{\D t}^{x,a}, X^{x,a}_{\D t}\big)} + r(x, a) \D t\Big],\ms\\
\dis\Pi(u,w,\theta)(x):=\frac{2\rho}{\sigma_0^2(x)} u(x)-\frac{2}{\sigma_0^2(x)}\biggl[ \lambda\ln\int_A \exp{\Bigl\{\frac{1}{\lambda}\bigl[\frac{1}{2}\sigma_1^2(x,a)\theta(x)+w(x,a)\bigr]\Bigr\}}da\biggr].
\ea\right.
\eea
%\begin{rem}

Note the above mappings only use the coefficient $\sigma$ which is assumed to be known (including its components $\sigma_0$ and $\sigma_1$) , as well as $\cX^{x},\cX^{x,a}$ which can be simulated and  $X^{x,a},r$ which are observed data processes upon queries.} %can be learned through trials and experiments.
%\end{rem}

Our algorithm is now built on the iteration of $(u,v,w,\theta)$. Similar to \reff{J},  introduce the new loss function
{ \bea\label{J-vol}
\dis J(u,v,w,\theta):=\norm{\Lambda(w,\th)-u}_0+\norm{\Phi(w,\th)-v}_0+\norm{\Psi(v)-w}_0+\|\Pi(u,w,\th)-\th\|_0.
\eea}
%As before, %Similar to the role of Proposition \ref{prop-fixedpoint}, we have
%the following result reduces our problem to minimizing  $J$.
\begin{prop}\label{prop-fixedpointvol}
%Under Assumption \ref{assum-vol}, t
There exists a constant $C_{L_*}>0$, depending on $L_0,L_1,L_*$, such that
\bea\label{Jestvol}
\dis J(u^*,v^*,w^*,\theta^*)
% =\norm{\Lambda-u}_0+\norm{\Phi-u}_0+\norm{\Psi-u}_0+\norm{\Pi-\theta}_0
\leq C_{L_*}\Delta t.
\eea
\end{prop}
All the proofs in this section become technically easy without considering the time discretization error and are postponed to Appendix. The following estimates are needed to show the convergence.

Recall the number $\eps$ in Assumption \ref{assum-vol}-(ii).
%{Recall Assumption \ref{assum-vol}-(ii).}
\begin{lem}\label{contr-vol}
%Under Assumption \ref{assum-vol},  t
There exists a constant $C>0$, depending only on $m,\lambda,L_0,L_1$ such that, for any $(u,v,w,\th),(u',v',w',\th')\in \cU\times\cV\times\cW\times\Theta$ and $(x,a)\in[0,\infty)\times A$, we have
\bea\label{volest1}
&&\norm{\D\Phi}_0+\norm{\D\Psi}_0\leq  \frac{C}{\sqrt{\rho}}\Bigl\{\varepsilon\norm{\D \th}_0+\norm{\D w}_0\Bigr\},\nonumber\\
&&\norm{\D\Lambda}_0\leq \frac{C}{\rho}\Bigl\{\varepsilon\norm{\D \th}_0+\norm{\D w}_0\Bigr\},\\
&&\|\D\Pi\|_0\leq C\Bigl\{\varepsilon\norm{\D \th}_0+\frac{1}{\sqrt{\rho}}\norm{\D w}_0\Bigr\},\nonumber
\eea
where $\D\Phi:=\Phi(w,\th)-\Phi(w',\th')$, $\D \th:=\th-\th'$, and similarly for $\D\Psi,\D\Lambda,\D\Pi$ and  $\D w$.
\end{lem}

We finally arrive at the main result of the section.
\begin{thm} \label{Jest-vol}
%Under Assumption \ref{assum-vol}, t
There exist constants $C, \rho_0,\eps_0$ depending on $m,\lambda,|A|$ and $L_0,L_1$ but not on $L_*$, such that for all $\rho>\rho_0$ and $\eps<\eps_0$,
\bea\label{volest2}
&&\norm{\Lambda(w,\th)-u^*}_0+\norm{\Phi(w,\th)-v^*}_0+\norm{\Psi(v)-w^*}_0+\|\Pi(u,w,\th)-\theta^*\|_0\nonumber\\
\leq&& CJ(u,v,w,\theta)+C_{L_*}\D t.
\eea
\end{thm}
Thus again, the problem boils down to solving the minimization problem
\beaa
\dis \inf_{u,v,w,\theta}J(u,v,w,\theta),
\eeaa
leading to the following algorithm.

\begin{breakablealgorithm}
\caption{(Algorithm for problems with controlled diffusion coefficients)}
\label{alg:vol-single}
\begin{algorithmic}[1]\small

\State \textbf{Input:} Fixed initial time $t_0=0$, truncation terminal time $T$,
number of grid points $K$, batch size $M$, nested path size $N$,
and epochs $E$; discount $\rho>0$, action space $A$, %action grid $\{a_j\}_{j=1}^{N_a}$,
temperature parameter $\lambda>0$,
softmax parameter $\tau$, and learning rates
$\eta_u,\eta_v,\eta_w,\eta_\theta$; distribution of the initial state
$\mathrm{Law}(X_0)$.

\State \textbf{Output:} Trained networks
$(u^\psi,v^\gamma,w^\phi,\theta^\xi)$
and the resulting estimated optimal value $\hat u$
and policy $\hat\pi$.

\State \textbf{Initialize:} Set
\[
\D t
\gets
\frac{T-t_0}{K-1},
\qquad
t_k
\gets
t_0+k\D t,
\qquad k=0,\dots,K-1,
\]
and randomly initialize the weights
$(\psi,\gamma,\phi,\xi)$ for
$u^\psi(x)$, $v^\gamma(x)$,
$w^\phi(x,a)$, and $\theta^\xi(x)$.

\For{$\mathrm{epoch}=1,\dots,E$}
    \For{$m=1,\dots,M$}
        \com{Batch generation}
        \State Draw
        $\cX_0^{(m)}
        \sim
        \mathrm{Law}(X_0)$.
        \State Draw primary Brownian increments
        $\{\D B_k^{(m)}\}_{k=0}^{K-2}
        \sim
        \cN(0,\D t)$.

        \State Set
        \[
        \nabla\cX_0^{(m)}\gets1,
        \qquad
        N_0^{(m)}\gets0.
        \]

        \For{$k=0$ \textbf{to} $K-2$}
            \com{Forward simulation of the primary reference path}

            \State
            \[
            \cX_{k+1}^{(m)}
            \gets
            \cX_k^{(m)}
            +
            \sigma_0(\cX_k^{(m)})
            \D B_k^{(m)}.
            \]

            \State
            \[
            \nabla\cX_{k+1}^{(m)}
            \gets
            \nabla\cX_k^{(m)}
            \exp\!\left(
                \sigma_0'(\cX_k^{(m)})\D B_k^{(m)}
                -
                \frac12
                \left|
                \sigma_0'(\cX_k^{(m)})
                \right|^2
                \D t
            \right).
            \]

            \State
            \[
            N_{k+1}^{(m)}
            \gets
            \frac{t_k-t_0}{t_{k+1}-t_0}
            N_k^{(m)}
            +
            \frac{1}{t_{k+1}-t_0}
            \nabla\cX_k^{(m)}
            \sigma_0^{-1}(\cX_k^{(m)})
            \D B_k^{(m)}.
            \]
        \EndFor

        \State For $k=0,\dots,K-1$, compute
        \[
        Q_k^{(m)}
        \gets
        \lambda
        \log
        \int_A
        \exp\!\left\{\frac{1}{\lambda}\left(
        \frac12
        \sigma_1^2(\cX_k^{(m)},a)
        \theta^\xi(\cX_k^{(m)})
        +
        w^\phi(\cX_k^{(m)},a)
        \right)\right\}da.
        \]

        \State Compute the path-level targets at
        $\cX_0^{(m)}$:
        \[
        \widetilde u^{(m)}
        \gets
        \sum_{k=1}^{K-1}
        e^{-\rho(t_k-t_0)}
        Q_k^{(m)}
        \D t,
        \]
        \[
        \widetilde v^{(m)}
        \gets
        \sum_{k=1}^{K-1}
        e^{-\rho(t_k-t_0)}
        N_k^{(m)}
        Q_k^{(m)}
        \D t,
        \]
        and
        \[
        \widetilde\theta^{(m)}
        \gets
        \frac{2\rho}{\sigma_0^2(\cX_0^{(m)})}
        u^\psi(\cX_0^{(m)})
        -
        \frac{2}{\sigma_0^2(\cX_0^{(m)})}
        Q_0^{(m)}.
        \]

        \For{$k=0$ \textbf{to} $K-2$}
            \com{Compute $w$-targets along the path}
            \For{$n=1,\dots,N$}
                \com{Nested transitions conditional on $a^{(n)}$}

                \State Draw
                $
                \D\hat B_k^{(m,n)}
                \sim
                \cN(0,\D t).
                $
                \State Sample
                $a^{(n)}
                \sim
                \mathrm{Unif}(A)$
                %\bigl(\{a_j\}_{j=1}^{N_a}\bigr)$.
                \State Apply the fixed action $a^{(n)}$ at
                $\cX_k^{(m)}$ and observe
                $X_{k+1}^{a^{(n)},(m,n)}$.

                \State Simulate the nested reference transition
                \[
                \cX_{k+1}^{(m,n)}
                \gets
                \cX_k^{(m)}
                +
                \sigma(\cX_k^{(m)},a^{(n)})
                \D\hat B_k^{(m,n)}.
                \]
%                \State Compute
%                \[
%                \begin{aligned}
%                W_k^{(m,n)}
%                \gets{}
%                \int_{\cX_k^{(m)}}^{
%                X_{k+1}^{a^{(n)},(m,n)}
%                }
%                v^\gamma(x)\,dx
%                -
%                \int_{\cX_k^{(m)}}^{
%                \cX_{k+1}^{(m,n)}
%                }
%                v^\gamma(x)\,dx
%                +
%                r(\cX_k^{(m)},a^{(n)})\D t.
%                \end{aligned}
%                \]
%            \EndFor
\State Compute
                \[
                \begin{aligned}
                W_k^{(m,n)}
                \gets{}
                {\cI\Big(v^\gamma; \cX_{k+1}^{(m,n)}, X_{k+1}^{a^{(n)},(m,n)}\Big)}
                +
                r(\cX_k^{(m)},a^{(n)})\D t.
                \end{aligned}
                \]
            \EndFor

            \State
            \[
            \widetilde w_k^{(m)}
            \gets
            \frac{1}{N\D t}
            \sum_{n=1}^N
            W_k^{(m,n)}.
            \]

            \State
            \[
            \ell_{w,k}^{(m)}
            \gets
            \left|
            w^\phi(\cX_k^{(m)},a^{(n)})
            -
            \widetilde w_k^{(m)}
            \right|.
            \]
        \EndFor

        \State Compute
        \[
        \ell_u^{(m)}
        \gets
        \left|
        u^\psi(\cX_0^{(m)})
        -
        \widetilde u^{(m)}
        \right|,
        \qquad
        \ell_v^{(m)}
        \gets
        \left|
        v^\gamma(\cX_0^{(m)})
        -
        \widetilde v^{(m)}
        \right|,
        \]
        \[
        \ell_\theta^{(m)}
        \gets
        \left|
        \theta^\xi(\cX_0^{(m)})
        -
        \widetilde\theta^{(m)}
        \right|,
        \qquad
        L_w^{(m)}
        \gets
        \mathrm{softmax}_\tau
        \left(
        \{\ell_{w,k}^{(m)}\}_{k=0}^{K-2}
        \right).
        \]
    \EndFor

    \State Aggregate the batch losses:
    \[
    L_u
    \gets
    \frac1M\sum_{m=1}^M\ell_u^{(m)},
    \qquad
    L_v
    \gets
    \frac1M\sum_{m=1}^M\ell_v^{(m)},
    \]
    \[
    L_\theta
    \gets
    \frac1M\sum_{m=1}^M\ell_\theta^{(m)},
    \qquad
    L_w
    \gets
    \frac1M\sum_{m=1}^M L_w^{(m)}.
    \]

    \State Set
    \[
    J
    \gets
    L_u+L_v+L_w+L_\theta.
    \]

    \State Update
    \[
    \psi
    \gets
    \psi-\eta_u\nabla_\psi J,
    \qquad
    \gamma
    \gets
    \gamma-\eta_v\nabla_\gamma J,
    \]
    \[
    \phi
    \gets
    \phi-\eta_w\nabla_\phi J,
    \qquad
    \xi
    \gets
    \xi-\eta_\theta\nabla_\xi J.
    \]
\EndFor

\State Set $\hat u\gets u^\psi$ and obtain $\hat\pi$
from $(w^\phi,\theta^\xi)$ through the
entropy-regularized optimizer.

\State \textbf{Return:}
$(u^\psi,v^\gamma,w^\phi,\theta^\xi,\hat u,\hat\pi)$.

\end{algorithmic}
\end{breakablealgorithm}

\begin{eg}\label{EXvol}
{\rm Consider an example where the dynamics coefficients are
\beaa
&&b(x,a)=xe^{-(x^2+a)},\qq
\sigma_0^2(x)=1,\qq
\sigma_1^2(x,a)=e^{-(x^2+a)},
\eeaa
and the running reward is
$$r(x,a)=(e^{-(x^2+a)}-2x^2+\rho+1)e^{-x^2}.$$
Setting $\lambda =1$, the HJB \reff{scalarHJB2} has the oracle solution along with its derivatives
\beaa
u^*(x)=e^{-x^2},\q v^*(x)=-2xu^*(x),\q \theta^*(x)=(4x^2-2)u^*(x).
\eeaa
Although the problem is in infinite horizon, for numerical implementation we terminate the time at $T=0.2$. We also set $x_0=0$, $\rho = 50, \D t=0.02$ in implementation. The result is given in the following table, which shows a somewhat similar, if not slightly less, accuracy compared to  the average accuracy in the previously reported examples for drift control problems.

\begin{table}[htbp]
  \centering
  \label{table-vol}
  \begin{tabular}{|c|c|c|c|c|}
    \hline
    \text{$u^*(x_0)$} &
    \text{Mean $\hat{u}(x_0)$} &
    \text{Mean RE} &
    \text{STD RE} &
    \text{Mean runtime (s)} \\
    \hline
    1.0000& 0.9330 & 6.700\% & 4.827\% &  58.06 \\
    \hline
%0.2 6 40

    \hline
%0.1 5 128 128 96 20
  \end{tabular}
  \caption{Numerical results for {Example }\ref{EXvol}.}
\end{table}}
\end{eg}
% Numerical example:
% \beaa
% u^*(x)&=&e^{-x^2},\\
% v^*(x)&=&-2xu^*(x),\\
% \theta^*(x)&=&(4x^2-2)u^*(x)\\
% b(x,a)&=&xe^{-2(x^2+a)}\\
% \sigma_0^2(x)&=&1,\\
% \sigma_1^2(x,a)&=&e^{-(x^2+a)},\\
% r(x,a)&=&(e^{-2(x^2+a)}-2x^2+\rho+1)u^*(x)
% \eeaa
% \begin{table}[htbp]
%   \centering
%   \label{xxxxxx}
%   \begin{tabular}{|c|c|c|c|c|}
%     \hline
%     \text{$u^*(t_0,x_0)$} &
%     \text{Mean $\hat{u}(x_0)$} &
%     \text{Mean RE} &
%     \text{STD RE} &
%     \text{Mean runtime (s)} \\
%     \hline
%     1.0000& 0.9330 & 6.700\% & 4.827\% &  58.06 \\
%     \hline
% %0.2 6 40

%     \hline
% %0.1 5 128 128 96 20
%   \end{tabular}
%   \caption{Numerical results for \textit{Example X }.}
% \end{table}
% Numerical example scaled:
% \beaa
% \dis u^*(x)&=&e^{-(\frac{x}{c})^2},\\
% \dis v^*(x)&=&-(\frac{2}{c^2})xu^*(x),\\
% \dis \theta^*(x)&=&(\frac{4}{c^4}x^2-\frac{2}{c^2})u^*(x)\\
% \dis b(x,a)&=&\frac{x}{c^2}e^{-\bigl((\frac{x}{c})^2+\frac{a}{c}\bigr)}\\
% \dis \sigma_0^2(x)&=&1,\\
% \dis \sigma_1^2(x,a)&=&e^{-\bigl((\frac{x}{c})^2+\frac{a}{c}\bigr)},\\
% \dis r(x,a)&=&(\frac{1}{c^2}e^{-\bigl((\frac{x}{c})^2+\frac{a}{c}\bigr)}-\frac{2}{c^4}x^2+\rho+\frac{1}{c^2})u^*(x)
% \eeaa

\section{Conclusions}\label{conclusion}

This paper studies potentially high-dimensional stochastic control problems with missing key primitives, and develops data-driven algorithms that learn the optimal value functions and optimal randomized policies with theoretical guarantees on convergence and convergence rates. The method also provides a scalable way to numerically solve certain black-box PDEs. It is distinctive from and complementary to the other approaches developed to solve similar model-free/black-box problems, such as stochastic approximation \cite{Tang-Zhou} and zeroth-order derivative \cite{JOPZ}.

There are clearly outstanding open questions going forward. One is to allow diffusion terms to depend on control in the general setting. Control entering into diffusion is important not only because this feature is inherent in many applications (e.g. portfolio selection), but also because this is the very feature that fundamentally distinguishes stochastic controls from determinist ones (see the many discussions on this point in Yong and Zhou \cite{YZ}). Technically, this feature demands an additional (and most likely delicate) analysis on Hessian, which we already had a glimpse in \S\ref{sect-volatility} for a very special case.

The other open question is when diffusion coefficients and/or terminal rewards are also unknown, which may arise in certain applications.  While such a case would invalidate our current approach right away, it is interesting to investigate a way around, namely to keep the same big idea of probabilistic representations of the derivatives but find a different mapping for the fix-point argument to work.

Finally, as discussed in Remark \ref{rem-data}, the assumption on the (unknown) environment in terms of how data are observed, Assumption \ref{data}-(i), can be strong for many applications. It is important to verify and make sense of this assumption in a given application. What is more interesting, however, is to investigate whether and how generative AI might help when the assumption fails.

\section{Appendix}
\label{sect-appendix}
\subsection{Proof of Proposition \ref{prop-fixedpoint}}
\proof
We first analyze $\Psi_K$. For any $(t_k, x, a)$, by  (\ref{vw*}) and (\ref{phimap}) we have
\beaa
%\label{DPsi}
&& (\Psi_K({v^*_{\hT_K}})-w^*_{\hT_K})(t_k,x,a)\\
&=& \frac{1}{\D t}\hE_{t_k,x}\Big[{\cI(v^*; t_k,x,X^{a}_{t_{k+1}}\big)} -{\cI(v^*; t_k,x,\cX_{t_{k+1}}\big)}\Big] - v^*(t_k,x) \cd b(t_k,x,a)\\
% &=&{\cI(v^*; t_k,\cX_{t_{k+1}},X^{a}_{t_{k+1}}\big)}\\
&=&\frac{1}{\D t}\hE_{t_k,x}\Big[\big[u^*(t_k,X^{ a}_{t_{k+1}}) - u^*(t_k, x)\big]-\big[ u^*(t_k,\cX_{t_{k+1}})-u^*(t_k,x)\big] \Big]-v^*(t_k,x)\cd b(t_k,x,a)\\
&=&\frac{1}{\D t}\hE_{t_k,x}\Big[\int_{t_k}^{t_{k+1}}\big[ {1\over 2} \si\si^\top : u^*_{xx}(t_k, X^{a}_s) +  b(t_k, X^{ a}_s, a) \cd v^*(t_k, X^{a}_s)\\
&&\qq\qq  - {1\over 2} \si\si^\top : u^*_{xx}(t_k, \cX_s)\big]ds\Big]-v^*(t_k,x)\cd b(t_k,x,a)\\
&=&\frac{1}{2\D t}\hE_{t_k,x}\Big[\int_{t_k}^{t_{k+1}}\big[\si\si^\top : u^*_{xx}(t_k, X^{a}_s)- \si\si^\top : u^*_{xx}(t_k, x)\big]ds\Big]\\
&& - \frac{1}{2\D t}\hE_{t_k,x}\Big[\int_{t_k}^{t_{k+1}}\big[ \si\si^\top : u^*_{xx}(t_k, \cX_s)- \si\si^\top : u^*_{xx}(t_k, x)\big]ds\Big]\\
&& + \frac{1}{\D t}\hE_{t_k,x}\Big[\int_{t_k}^{t_{k+1}} \big[ v^*(t_k, X^{a}_s)\cd b(t_k, X^{ a}_s, a) -v^*(t_k,x)\cd b(t_k,x,a)\big] ds\Big].
\eeaa
Then, by Assumption \ref{assum-standing} and Lemma \ref{lem-Kerr},
\bea
\label{DPsi}
\Big| (\Psi_K({v^*_{\hT_K}})-w^*_{\hT_K})\Big|
\le C_{L_*} \sup_{t_k\le s\le t_{k+1}} \hE\Big[ |X^{a}_s-x| + |\cX_s-x|\big]ds\Big]\le C_{L_*, T}\sqrt{\D t}.
\eea

We next analyze $\Phi_K$. Recall \reff{H} and denote
\bea
\label{H*}
H^*(t,x) := H(t,x, u^*_x(t,x)).
\eea
 For any $(t_k, x)$, and assuming without loss of generality that $k=0$, by  (\ref{v*rep}) and (\ref{phimap}) we have
\bea
\label{PhiKest}
&&\dis v^*(0,x) - \Phi_K(w^*_{\hT_K})(0, x) = \e^{\D t}_K +  \e^{\D t}_0 + \sum_{j=1}^{K-1} [\eps_j^{1,\D t}+ \eps_j^{2,\D t}],\q\mbox{where}\\
&&\dis \eps_K^{\D t}:= \hE_{0,x}\Big[(\nabla \cX_T)^\top g_x(\cX_T)-(\nabla \cX_T^{\D t})^\top g_x(\cX_T^{\D t})\Big],\nonumber\\
&&\dis \eps_0^{\D t}:= \hE_{0,x}\Big[\int_{t_k}^{t_{k+1}} H^*(s, \cX_s)  N_s ds\Big], \nonumber\\
&&\dis \eps_j^{1,\D t} := \hE_{0,x}  \Big[\int_{t_{j}}^{t_{j+1}}\big[H^*(s, \cX_s)- H^*(t_j, \cX^{\D t}_{t_j})  \big] N_s ds\Big], \nonumber\\
&& \dis \eps_j^{2,\D t} := \hE_{0,x}\Bigl[\int_{t_{j}}^{t_{j+1}} H^*(t_j,\cX_{t_j}^{\D t})\big[ N_{s} - N_{t_j}^{\D t}\big]ds\Big],\nonumber
\eea
for $j=1, \cds, K-1$. By \reff{H} one can easily verify that
\bea
\label{H*est}
\|H^*\|_0 \le  C_{L_*},\q |H^*(t,x)-H^*(t',x')|\le C_{L_*}\big[\sqrt{|t-t'|}+|x-x'|\big].
\eea
Then, applying Lemma \ref{lem-Kerr}, we have
\bea
\label{PhiKest1}
\left.\ba{lll}
\dis | \eps_K^{\D t}| \le C_{L_*}\dbE_{0,x}\Big[ |\nabla \cX_T - \nabla \cX_T^{\D t}| +  |\td \cX_T| |\cX_T - \cX_T^{\D t}|\Big] \le C_{L_*,T} \sqrt{\D t},\ms\\
\dis | \eps_0^{\D t}| \le C_{L_*, T}\dbE\Big[\int_{t_0}^{t_{1}}  |N_s| ds\Big] \le C_{L_*, T}\int_{t_0}^{t_{1}} {1\over \sqrt{s}} ds = C_{L_*, T} \sqrt{\D t}.
\ea\right.
\eea
Moreover, note that
\beaa
|\eps_j^{1,\D t}| &\le& C_{L_*} \hE_{0,x}  \Big[\int_{t_{j}}^{t_{j+1}}\big[\sqrt{s-t_j} + |\cX_s - \cX_{t_j}|+ |\cX_{t_j} - \cX^{\D t}_{t_j}|\big] |N_s| ds\Big]\\
&\le& C_{L_*} \int_{t_{j}}^{t_{j+1}}\Big(\hE_{0,x} \big[s-t_j + |\cX_s - \cX_{t_j}|^2+ |\cX_{t_j} - \cX^{\D t}_{t_j}|^2\big]\Big)^{1\over 2} \Big(\hE_{0,x}[|N_s|^2]\Big)^{1\over 2} ds\\
&\le& C_{L_*,T}\sqrt{\D t}\int_{t_{j}}^{t_{j+1}} {1\over \sqrt{ s}}ds.
\eeaa
Then
\bea
\label{PhiKest2}
\sum_{j=1}^{K-1} |\eps_j^{1,\D t}| \le \sum_{j=1}^{K-1}C_{L_*,T}\sqrt{\D t}\int_{t_{j}}^{t_{j+1}} {1\over \sqrt{ s}}ds =  C_{L_*,T}\sqrt{\D t}\int_{t_1}^{T} {1\over \sqrt{ s}}ds \le C_{L_*,T}\sqrt{\D t}.
\eea

It remains to estimate $\eps_j^{2,\D t}$. Denote
\beaa
I_l :=  ( \si^{-1}(l, \cX_l)\td \cX_l)^\top,\q I^{\D t}_l := \sum_{i=0}^{K-1} \1_{[t_i, t_{i+1})}(l) ( \si^{-1}(t_i, \cX^{\D t}_{t_i})\td \cX^{\D t}_{t_i})^\top.
\eeaa
Then, noting that ${\D t\over t_j} = {1\over j}$, we have
\beaa
\eps_j^{2,\D t} &=& \hE_{0,x}\Bigl[\int_{t_{j}}^{t_{j+1}} H^*(t_j,\cX_{t_j}^{\D t})\big[ \dbE_{0,x}[N_{s}|\cF_{t_j}] - N_{t_j}^{\D t}\big]ds\Big]\\
&=&\hE_{0,x}\Big[\int_{t_{j}}^{t_{j+1}} H^*(t_j,\cX_{t_j}^{\D t})\big[ {1\over s} \int_0^{t_j} I_l dB_l- {\D t\over t_j} \int_0^{t_j} I^{\D t}_l dB_l\big]ds\Big]\\
&=& \hE_{0,x}\Big[H^*(t_j,\cX_{t_j}^{\D t})\big[ \ln {t_{j+1}\over t_j} \int_0^{t_j} I_l dB_l- {\D t\over t_j} \int_0^{t_j} I^{\D t}_l dB_l\big]\Big]\\
&=& \hE_{0,x}\Big[H^*(t_j,\cX_{t_j}^{\D t})\big[{1\over j} \int_0^{t_j} (I_l-I^{\D t}_l) dB_l+ c_j \int_0^{t_j} I_l dB_l\big]\Big],
\eeaa
where $c_j:= \ln {t_{j+1}\over t_j} - {1\over j} = \ln (1+{1\over j})-{1\over j}$ satisfies $|c_j|\le {C\over j^2}$. Accordingly, by \reff{H*est},
\beaa
|\eps_j^{2,\D t}|^2 &\le& {C_{L_*}\over j^2} \hE_{0,x}\Big[ \int_0^{t_j} |I_l-I^{\D t}_l|^2 dl \Big] + {C\over j^4}\hE_{0,x}\Big[ \int_0^{t_j} |I_l|^2 dl \Big]\\
&\le&  {C_{L_*}\over j^2} \sum_{i=0}^{j-1}\int_{t_i}^{t_{i+1}} \hE_{0,x}\big[ |I_l-I^{\D t}_l|^2\big] dl + {Ct_j\over j^4}.
\eeaa
For $l\in [t_i, t_{i+1}]$, it follows easily from Assumption \ref{assum-standing} and Lemma \ref{lem-Kerr} that
\beaa
\hE_{0,x}\big[ |I_l-I^{\D t}_l|^2\big] \le C\hE_{0,x}\Big[ |I_l-I_{t_i}|^2 + |I_{t_i}-I^{\D t}_{t_i}|^2\Big]\le C_{L_*,T}\D t.
\eeaa
Then
\beaa
|\eps_j^{2,\D t}|^2 &\le&  {C_{L_*,T}\over j^2} \sum_{i=0}^{j-1} |\D t|^2 + {C\D t\over j^3} \le {C_{L_*,T}|\D t|^2 \over j}.
\eeaa
Thus
\bea
\label{PhiKest3}
\sum_{j=1}^{K-1} |\eps_j^{2,\D t}| \le C_{L_*,T} \D t\sum_{j=1}^{K-1}{1 \over \sqrt{j}} \le C_{L_*,T}\D t \sqrt{K} =  C_{L_*,T}\sqrt{T\D t}.
\eea
Plugging \reff{PhiKest1}, \reff{PhiKest2} and \reff{PhiKest3} into \reff{PhiKest}, we see that
\beaa
|v^*(0,x) - \Phi_K(w^*_{\hT_K})(0, x)| \le C_{L_*,T}\sqrt{\D t}.
\eeaa
This, together with \reff{DPsi} and the arbitrariness of $(t_k,x,a)$, proves \reff{PhiPsiDt}.
\qed
\subsection{Proof of Proposition \ref{prop-fixedpointvol}}
 \proof The proof is much simplified once we omit the time discretization. We can directly read from the definition \reff{volmaps} that for any $x$,
 \beaa
 \Phi(w^*,\th^*)(x)=v^*(x),\qq\Lambda(w^*,\th^*)(x)=u^*(x),\qq\Pi(u^*,w^*,\theta^*)=\th^*(x).
 \eeaa
 For arbitrary $(x,a)$,
 \beaa
\Bigl|\bigl(\Psi(v^*)-w^*\bigr)(x,a)\Bigr|&=&\Bigl|\frac{1}{\D t}\hE\Bigl[\bigl(\int_0^{\D t}b(X_{t}^{x,a},a)v^*(X_t^{x,a})-b(x,a)v^*(x)dt\bigr) \Bigr]\Bigr|.
 \eeaa
 Thus by Assumption \ref{assum-vol},
 \beaa
 \Bigl|\bigl(\Psi(v^*)-w^*\bigr)(x,a)\Bigr|\leq C_{L_*}\sup_{0<t<\D t}\hE\Bigl[|X_t^{x,a}-x|\Bigr]\leq C_{L_*}\D t.
 \eeaa
 The above two estimates, together with the arbitrariness of $(x,a)$, complete the proof.
 \qed

\subsection{Proof of Lemma \ref{contr-vol}}
\proof For arbitrary $(x,a)$, we first have
\beaa
&&\lambda\ln\int_A \exp{\Bigl\{\frac{1}{\lambda}\bigl[\frac{1}{2}\sigma_1^2(\cX^x_t,a)\theta(\cX^x_t)+w(\cX^x_t,a)\bigr]\Bigr\}}da\\
\leq&&\lambda\ln\int_A \exp{\Bigl\{\frac{1}{\lambda}\bigl[\frac{1}{2}\sigma_1^2(\cX^x_t,a)\bigl(\th'(\cX^x_t)+\norm{\th-\th'}_0)\bigr)+w'(\cX^x_t,a)+\norm{w-w'}_0\bigr]\Bigr\}}da\\
=&&\lambda\ln\int_A \exp{\Bigl\{\frac{1}{\lambda}\bigl[\frac{1}{2}\sigma_1^2(\cX^x_t,a)\theta'(\cX^x_t)+w'(\cX^x_t,a)\bigr]\Bigr\}}da +\frac{\eps}{2}\norm{\th-\th'}+\norm{w-w'}_0.
\eeaa
This leads to
\bea \label{gamma-est-vol}
&&\Bigl|\lambda\ln\int_A e^{\frac{1}{\lambda}\bigl[\frac{1}{2}\sigma_1^2(\cX^x_t,a)\theta(\cX^x_t)+w(\cX^x_t,a)\bigr]}da-\lambda\ln\int_A e^{\frac{1}{\lambda}\bigl[\frac{1}{2}\sigma_1^2(\cX^x_t,a)\theta(\cX^x_t)+w(\cX^x_t,a)\bigr]}da\Bigr| \nonumber\\
\leq&& \frac{\eps}{2}\norm{\th-\th'}+\norm{w-w'}_0.
\eea
It follows from \reff{volmaps} and Lemma \reff{Kerr1} that
\bea\label{v-est-vol}
|\D \Phi(x)|&\leq& C\Bigl\{\eps\norm{\th-\th'}+\norm{w-w'}_0\Bigr\}\hE\Bigl[\int_0^\infty e^{-\rho t}|N_t^x|dt\Bigr]\nonumber\\
&\leq&\frac{C}{\sqrt{\rho}}\Bigl\{\varepsilon\norm{\th-\th'}_0+\norm{w-w'}_0\Bigr\}.
\eea
Then, by \reff{v-est-vol}, \reff{volmaps} and the arbitrariness of $(x,a)$, we  have
\bea\label{w-est-vol}
\norm{\D \Psi}_0\leq C\norm{\D \Phi}_0\leq \frac{C}{\sqrt{\rho}}\Bigl\{\varepsilon\norm{\th-\th'}_0+\norm{w-w'}_0\Bigr\}.
\eea
Next, by \reff{volmaps} and \reff{gamma-est-vol}, we obtain similarly
\bea\label{u-est-vol}
|\D \Lambda(x)|&\leq&C\Bigl\{\eps\norm{\th-\th'}+\norm{w-w'}_0\Bigr\}\hE\Bigl[\int_0^\infty e^{-\rho t}dt\Bigr]\nonumber\\
&\leq&\frac{C}{\rho}\Bigl\{\varepsilon\norm{\th-\th'}_0+\norm{w-w'}_0\Bigr\}.
\eea
Combining  \reff{volmaps}, \reff{w-est-vol} and \reff{u-est-vol}, together with the arbitrariness of $(x,a)$, we obtain
\beaa
\norm{\D \Pi}_0&\leq& C\Bigl\{\rho\norm{\D \Lambda}_0+ \eps\|\D \Pi\|_0+\norm{\D \Psi}_0\Bigr\}\\
&\leq& C\Bigl\{\eps \norm{\th-\th'}_0+\norm{\D \Psi}_0 \Bigr\}\\
&\leq& C\Bigl\{\varepsilon\norm{\th-\th'}_0+\frac{1}{\sqrt{\rho}}\norm{w-w'}_0\Bigr\}.
\eeaa
\qed

\subsection{Proof of Theorem \ref{Jest-vol}}
\proof Fix $\rho_0,\eps_0$ which will be specified later, and assume without loss of generality that $\rho_0>1$ and $\eps_0<1$. Applying Proposition \ref{prop-fixedpointvol} and Lemma \ref{contr-vol} we have
\bea\label{util*}
\norm{\Lambda(w,\th)-u^*}_0&\leq&\norm{\Lambda(w,\th)-\Lambda(w^*,\th^*)}_0+\norm{\Lambda(w^*,\th^*)-u^*}_0\nonumber\\
&\leq& \frac{C}{\rho_0}\Bigl\{\eps_0\norm{\th-\th^*}_0+\norm{w-w^*}_0\Bigr\}\nonumber\\
&\leq& C\Bigl\{\eps_0\Bigl(\norm{\th-\Pi(u,w,\th)}_0+\norm{\Pi(u,w,\th)-\th^*}_0\Bigr)\nonumber\\
&&+\frac{1}{\sqrt{\rho_0}}\Bigl(\norm{w-\Psi(w,\th)}_0+\norm{\Psi(w,\th)-w^*}_0\Bigr)\Bigr\}\nonumber\\
&\leq& C\Bigl\{\eps_0\Bigl(J(u,v,w,\th)+\norm{\Pi(u,w,\th)-\th^*}_0\Bigr)\nonumber\\
&&+\frac{1}{\sqrt{\rho_0}}\Bigl(J(u,v,w,\th)+\norm{\Psi(w,\th)-w^*}_0\Bigr)\Bigr\}.
\eea
Similarly, we have the following estimates for the other three terms
\bea\label{vtil*}
\norm{\Phi(w,\th)-v^*}_0&\leq& \norm{\Phi(w,\th)-\Phi(w^*,\th^*)}_0+\norm{\Phi(w^*,\th^*)-v^*}_0\nonumber\\
&\leq& \frac{C}{\sqrt{\rho_0}}\Bigl\{\eps_0\norm{\th-\th^*}_0+\norm{w-w^*}_0\Bigr\}\nonumber\\
&\leq& C\Bigl\{\eps_0\Bigl(\norm{\th-\Pi(u,w,\th)}_0+\norm{\Pi(u,w,\th)-\th^*}_0\Bigr)\nonumber\\
&&+\frac{1}{\sqrt{\rho_0}}\Bigl(\norm{w-\Psi(w,\th)}_0+\norm{\Psi(w,\th)-w^*}_0\Bigr)\Bigr\}\nonumber\\
&\leq& C\Bigl\{\eps_0\Bigl(J(u,v,w,\th)+\norm{\Pi(u,w,\th)-\th^*}_0\Bigr)\nonumber\\
&&+\frac{1}{\sqrt{\rho_0}}\Bigl(J(u,v,w,\th)+\norm{\Psi(w,\th)-w^*}_0\Bigr)\Bigr\},
\eea
\bea\label{thtil*}
\norm{\Pi(u,w,\th)-\th^*}_0&\leq& \norm{\Pi(u,w,\th)-\Pi(u^*,w^*,\th^*)}_0+\norm{\Pi(u^*,w^*,\th^*)-\th^*}_0\nonumber\\
&\leq& C\Bigl\{\eps_0\norm{\th-\th^*}_0+\frac{1}{\sqrt{\rho_0}}\norm{w-w^*}_0\Bigr\}\nonumber\\
&\leq& C\Bigl\{\eps_0\Bigl(\norm{\th-\Pi(u,w,\th)}_0+\norm{\Pi(u,w,\th)-\th^*}_0\Bigr)\nonumber\\
&&+\frac{1}{\sqrt{\rho_0}}\Bigl(\norm{w-\Psi(w,\th)}_0+\norm{\Psi(w,\th)-w^*}_0\Bigr)\Bigr\}\nonumber\\
&\leq& C\Bigl\{\eps_0\Bigl(J(u,v,w,\th)+\norm{\Pi(u,w,\th)-\th^*}_0\Bigr)\nonumber\\
&&+\frac{1}{\sqrt{\rho_0}}\Bigl(J(u,v,w,\th)+\norm{\Psi(w,\th)-w^*}_0\Bigr)\Bigr\},
\eea
and
\bea\label{wtil*}
\norm{\Psi(v)-w^*}_0&\leq& \norm{\Psi(v)-\Psi(v^*)}_0+\norm{\Psi(v^*)-w^*}_0\nonumber\\
&\leq& \frac{C}{\sqrt{\rho_0}}\Bigl\{\eps_0\norm{\th-\th^*}_0+\norm{w-w^*}_0\Bigr\}+C_{L_*}\D t\nonumber\\
&\leq& C\Bigl\{\eps_0\Bigl(\norm{\th-\Pi(u,w,\th)}_0+\norm{\Pi(u,w,\th)-\th^*}_0\Bigr)\nonumber\\
&&+\frac{1}{\sqrt{\rho_0}}\Bigl(\norm{w-\Psi(w,\th)}_0+\norm{\Psi(w,\th)-w^*}_0\Bigr)\Bigr\}+C_{L_*}\D t\nonumber\\
&\leq& C\Bigl\{\eps_0\Bigl(J(u,v,w,\th)+\norm{\Pi(u,w,\th)-\th^*}_0\Bigr)\nonumber\\
&&+\frac{1}{\sqrt{\rho_0}}\Bigl(J(u,v,w,\th)+\norm{\Psi(w,\th)-w^*}_0\Bigr)\Bigr\}+C_{L_*}\D t.
\eea
Combining \reff{util*}, \reff{vtil*}, \reff{thtil*} and \reff{wtil*}, similar to the derivation of \reff{DPhi*2}, we have
\beaa
&&\norm{\Lambda(w,\th)-u^*}_0+\norm{\Phi(w,\th)-v^*}_0+\norm{\Psi(v)-w^*}_0+\|\Pi(u,w,\th)-\theta^*\|_0 \nonumber \\
 \leq &&\neg\neg\neg\neg\neg\neg\neg\neg\neg C\Bigl\{\eps_0\Bigl(J(u,v,w,\th)+\norm{\Pi(u,w,\th)-\th^*}_0\Bigr)\nonumber
+\frac{1}{\sqrt{\rho_0}}\Bigl(J(u,v,w,\th)+\norm{\Psi(w,\th)-w^*}_0\Bigr)\Bigr\}+C_{L_*}\D t.
\eeaa
Set $\eps_0:=\frac{1}{4C}\wedge 1$ and $\rho_0:=\frac{1}{16C^2}\vee 1$ for the above $C$. Then
\beaa
&&\norm{\Lambda(w,\th)-u^*}_0+\norm{\Phi(w,\th)-v^*}_0+\norm{\Psi(v)-w^*}_0+\|\Pi(u,w,\th)-\theta^*\|_0\nonumber \\
\leq&&\frac{C}{2}J(u,v,w,\th)+\frac{1}{4}\Bigl(\norm{\Pi(u,w,\th)-\th^*}_0+\norm{\Psi(w,\th)-w^*}_0\Bigr)+C_{L_*}\D t.
\eeaa
This implies the desired result immediately.
\qed

\bibliographystyle{abbrv}
\bibliography{references_261111}

\end{document}